%% file: main.tex
\documentclass{article}
\usepackage[accepted]{aistats2020}
\usepackage[utf8]{inputenc}
\usepackage[T1]{fontenc}
\usepackage{dsfont}

\usepackage{amsthm}
\usepackage{amsmath}
\usepackage{amsfonts}
\usepackage{algorithm}
\usepackage{algpseudocode}
\usepackage{fancyhdr}
\usepackage{bm}
\usepackage{subfig}
\usepackage{graphicx}
\usepackage{booktabs}
\usepackage{verbatim}
\usepackage{xcolor}
\usepackage{thmtools} 
\usepackage{thm-restate}
\usepackage{wrapfig}
\usepackage{url}

\newtheorem{problem}{Problem}

\newcommand{\notindependent}{\not\!\perp\!\!\!\perp}
\usepackage[round]{natbib}

\begin{document}
%

%

\newcommand{\bs}{\textrm{BP}}
\newcommand{\nbs}{\textrm{NBP}}
\newcommand{\prob}{\mathbb{P}}
\newcommand{\jlsay}[1]{[\textcolor{blue}{\textbf{JL: }}\textcolor{blue!60!black}{#1}]}
\newcommand{\wjsay}[1]{}

\setlength{\abovecaptionskip}{\baselineskip}
\newcommand{\ourtitle}{More Powerful Selective Kernel Tests for Feature Selection}
\twocolumn[

\aistatstitle{\ourtitle}
\runningauthor{Lim, Yamada, Jitkrittum, Terada, Matsui, Shimodaira}
\aistatsauthor{Jen Ning Lim \\
University College London
\And Makoto Yamada \\
Kyoto University, RIKEN AIP
\And 
Wittawat Jitkrittum \\
MPI for Intelligent Systems, T\"{u}bingen
\AND Yoshikazu Terada \\
Osaka University, RIKEN AIP
\And Shigeyuki Matsui \\
Nagoya University\And Hidetoshi Shimodaira \\
Kyoto University, RIKEN AIP}

\aistatsaddress{}
]
\begin{abstract}
Refining one's hypotheses in the light of data is a common scientific practice; however,
the dependency on the data introduces selection bias and can lead to specious statistical
analysis. 
An approach for addressing this is via conditioning on the selection
procedure to account for how we have used the data to generate our
hypotheses, and prevent
information to be used again after selection.
Many selective inference (a.k.a. post-selection inference) algorithms typically take this approach but will ``over-condition''
for sake of tractability. While this practice yields well calibrated statistic tests with controlled false positive rates (FPR),
it can incur a major loss in power. 
In our work, we extend two recent proposals for selecting features using the Maximum Mean Discrepancy
and Hilbert Schmidt Independence Criterion to condition on the minimal
conditioning event. We show how recent advances in
multiscale bootstrap makes
conditioning on the minimal selection event possible and demonstrate our proposal over a range of synthetic and real world experiments.
Our results show that our proposed test is indeed more powerful in most scenarios.
\end{abstract}

\section{INTRODUCTION}
%

Most statistical methods implicitly assume that parameters of the statistical investigation
are fixed apriori; that is, the choice of model, hypothesis to test, and parameters to be
estimated do not change before the data is inspected. Failure to satisfy this
can lead to disturbing properties such as uncalibrated $p$-values \citep{simmons2011false, gelman2013garden}.
The field of
selective inference (SI) considers a modernised version of statistical analysis
where we first explore the data and determine relevant parameters for our investigation.
Then, SI aims to provide valid inference under the model chosen by the data \citep{fithian2014optimal}. In our work, we extend two
algorithms that first select a set of features then perform hypothesis
testing on each of the selected features to determine whether it is
statistically significant.

One of the approaches in the field of SI is conditioning
on how the data has been used during the initial selection phase (\citet{fithian2014optimal, lee2016exact, fithian2015selective}).
This approach may be difficult to use 
since
it requires an explicit characterisation of
the selection procedure and the conditional distribution of the test statistic,
both of which can be
difficult to obtain. Fortunately, one of the key developments that has allowed
many SI algorithms to be tractable is the polyhedral
lemma \citep{lee2016exact, tibshirani2016exact}. Assume that the test
statistic is normally distributed before the selection. The polyhedral lemma
states that if the selection event can be written as a set of linear
constraints, then its conditional post-selective distribution follows a
truncated normal distribution \citep[Theorem 5.2]{lee2016exact}. This
result has been successfully applied to non-parametric kernel methods
for selecting informative features using the Hilbert Schmidt Independence
Criterion \citep{yamada2016post}, and the Maximum Mean Discrepancy
\citep{yamada2018post}, as well as multiple model comparison
\citep{lim2019kernel}.

A subtlety with SI is that the power of selective hypothesis tests (i.e.,
tests with null hypothesis that is determined by data and so \textit{random})
depends upon our choice of what to condition on.
If we condition on too little, the test will have uncontrolled false positive rate. If we condition on too much, it can incur a loss of power
\citep{fithian2014optimal}. This observation has driven research efforts to curate more
powerful hypothesis tests that have higher ``left-over'' information. These
proposals include randomising the data used during selection
\citep{tian2018selective}; a careful characterisation of how the data has
been used during the selection process as there are different costs for variable
selection and target formulation \citep{liu2018more}; and conditioning on the
minimal set, i.e., condition only on what is necessary for the test to be
valid but not more \citep{liu2018more, terada2019selective}.
The last idea forms the basis of what we propose in this work.


In the present work, we consider the problem of selecting a subset of
informative features with selective inference. We study two related
problem settings. In the first setting (Section \ref{sec:si_mmd}), given two
samples, the goal is to select a subset of features for which the marginal
distributions (restricted to the selected subset) of the two underlying
distributions significantly differ. In the second setting (Section \ref{sec:si_hsic}),
given a joint sample of covariate and response variables, the goal is to
select a subset of covariate variables whose dependency on the response is
statistically significant. While the selective tests of
\citet{yamada2018post,yamada2016post} are applicable to these problems, and have a tractable null distribution due
to the use of the polyhedral lemma, these tests do not consider the minimal
condition set, meaning that the tests may be overly conservative. 

We generalize the tests of \citet{yamada2018post} and
\citet{yamada2016post}, for the two settings respectively, and propose tests that condition on the minimal conditioning set by using the
selective multiscale bootstrap \citep{terada2017selective,
terada2019selective}. For the second problem, we
further propose a new estimator for the Hilbert Schmidt Independence
Criterion (HSIC) that takes the form of an incomplete U-statistic. We show
that the new estimator leads to a test that has higher power than the test of \citet{yamada2016post} which relies on the block estimator.
In experiments (Section \ref{sec:experiments}) on both synthetic and real
problems, we show that the new tests have well-controlled false positive
rate, and are more powerful than their respective original tests when the
number of features is large, and the number of selected features is larger
than one.



%
\input{background.tex}
%
\input{proposal_mmd.tex}
\input{proposal_hsic.tex}
\input{experiment.tex}
%
\input{disc.tex}
\newpage
\subsubsection*{Acknowledgements}
 M.Y. was supported by the
JST PRESTO program JPMJPR165A and partly supported by MEXT KAKENHI 16H06299 and  the RIKEN engineering network funding. S.M. was supported by MEXT
KAKENHI 16H06299.
\bibliography{ref}

\clearpage

\onecolumn
\appendix

\begin{center}
{\LARGE{}{}{}{}{}{}\ourtitle{}} 
\par\end{center}

\begin{center}
\textcolor{black}{\Large{}{}{}{}{}{}Supplementary}{\Large{}{}{}{}{}
} 
\par\end{center}

\input{appendix.tex}

\end{document}

%% file: background.tex
\section{BACKGROUND}
%
In this section, we review the Maximum Mean Discrepancy (MMD) 
and Hilbert Schmidt Independence Criterion (HSIC) which are used as
our criteria to select features as well as briefly introduce the concept of multiscale bootstrap. In Section \ref{sec:si_mmd}, we use MMD to select features that have significantly different marginal distributions and in Section \ref{sec:si_hsic}, we use HSIC to select features which have a significant dependence on the response variable respectively.

\textbf{Maximum Mean Discrepancy (MMD)} For a distribution $P$ and a positive definite kernel $K$, the mean embedding of $P$ is
defined as $\mu_P = \mathbb{E}_{x\sim P} [K(\cdot, x)]$ \citep{SmoGreSonSch2007}. The Maximum Mean Discrepancy (MMD) is a pseudo metric between two distributions
$P$ and $Q$ and is defined as $\mathrm{MMD}(P,Q) = \| \mu_P - \mu_Q\|_K$, where $\|\cdot \|_K$ denotes the norma in the reproducing kernel Hilbert space (RKHS) associated with $K$. If $K$ is a characteristic kernel, then $\mathrm{MMD}^2(P,Q) = 0 \iff P=Q$  \citep{gretton2012kernel}. An example of a characteristic kernel is the 
Gaussian kernel. It can be shown that the squared \textrm{MMD} can be written equivalently as 
$\mathrm{MMD}^2(P,Q) = \mathbb{E}_{z,z'\sim P\times Q}[h(z, z')]$ where $z:=(x,y)$ and
$h(z, z') := K(x,x') + K(y,y') - K(x',y) - K(x, y')$.
Given samples $\bm{x} := \{x_i\}_{i=1}^n$ and $\bm{y} := \{y_i\}_{i=1}^n$ of size $n$ as i.i.d.\ draws from 
$P$ and $Q$ respectively, 
an unbiased estimator is the U-statistic
$\widehat{\mathrm{MMD}}^2_u(\bm{x},\bm{y})=\frac{1}{n(n-1)}\sum_{i\neq j} h(z_i,
z_j)$. 
In our
work, we focus on a parametric bootstrap resampling procedure for multiscale
bootstrap and thus we use estimators with normal asymptotic distributions 
such as the
linear-time estimator $\widehat{\mathrm{MMD}}^2_l = \frac{2}{n}\sum_{i=1}^{n/2}
h(z_{2i}, z_{2i-1})$ \citep{gretton2012kernel}, and the incomplete U-statistic estimator
\citep{blom1976some, janson1984asymptotic} proposed by \cite{yamada2018post}:
$\widehat{\mathrm{MMD}}^2_{Inc} = \frac{1}{|\mathcal{D}_n|} \sum_{(i,j) \in
\mathcal{D}_n} h(z_i, z_j)$ where $\mathcal{D}_n$ is
random and sampled with replacement from $\{(i,j) \}_{i\neq j}$. Under weak assumptions, both $\widehat{\mathrm{MMD}}^2_l$ and
$\widehat{\mathrm{MMD}}^2_{Inc}$ are asymptotically normal for both when $P=Q$ and $P\neq Q$ 
 \citep{gretton2012kernel, yamada2018post}.

\newcommand\independent{\protect\mathpalette{\protect\independenT}{\perp}}
\def\independenT#1#2{\mathrel{\rlap{$#1#2$}\mkern2mu{#1#2}}}
\textbf{Hilbert Schmidt Independence Criterion (HSIC)} Let $x \in \mathcal{X}$
and $y \in \mathcal{Y}$ be two random variables with joint distribution
$P_{xy}$ and their respective marginals $P$ and $Q$. Let $K_\mathcal{X}$
and $K_\mathcal{Y}$ be two real-valued kernel functions defined on $\mathcal{X} \times \mathcal{X}$ and $\mathcal{Y} \times \mathcal{Y}$ respectively. The 
Hilbert Schmidt Independence Criterion \citep{gretton2005measuring} is defined as the Hilbert-Schmidt norm
of the covariance operator $\mathrm{HSIC}(P_{xy}) = || \mu_{xy} - \mu_P \otimes \mu_Q ||^2_{\mathrm{HS}}$
where $\otimes$ denotes the tensor product and 
$\mu_{xy} := \mathbb{E}_{(x,y) \sim P_{xy}} [K_\mathcal{X}(x, \cdot) \otimes K_\mathcal{Y}(y, \cdot)]$. The norm $||\cdot||_{\mathrm{HS}}$ is induced by the inner product of the space of linear operators (that are Hilbert Schmidt). See \citet[Section 2]{gretton2005measuring} for details.
If the product kernel $K_\mathcal{X}K_\mathcal{Y}$ is characteristic on the joint domain  $\mathcal{X} \times \mathcal{Y}$,
then $\mathrm{HSIC}(P_{xy}) = 0 \iff x \independent y$
($x$ and $y$ are independent) \citep[Theorem 3]{fukumizu2008kernel}. An example
of such a kernel can be constructed by letting $K_\mathcal{X}$ and $K_\mathcal{Y}$ be Gaussian
kernels on $\mathcal{X} \subseteq \mathbb{R}^{d_x}$ and $\mathcal{Y} \subseteq \mathbb{R}^{d_y}$ respectively.
Given $\bm{z}:=\{(x_i,y_i)\}_{i=1}^n$ consisting of $n$ i.i.d.\ samples from $P_{xy}$, an unbiased estimator 
can be computed as a U-statistic
$\widehat{\mathrm{HSIC}}_u(\bm{z}) = \frac{(n-4)!}{n!}\sum_{(i,j,q,r) \in \bm{i}^n_4} h(i,j, q,r)$ \citep{hoeffding1992class, song2012feature}
where $\bm{i}^n_4$ is the set of all $4$-tuples with each index occurring only once, 
$h(i,j,q,r) = \frac{1}{4!}\sum_{(s,t,u,v)}^{(i,j,q,r)}\bm{K}_{st}[\bm{L}_{st}+\bm{L}_{uv}-2\bm{L}_{su}]$ is the U-statistic kernel
with the sum being over $4!$ quadruples $(s,t,u,v)$ as permutations of $(i,j,q,r)$,
and $\bm{K,L} \in \mathbf{R}^{n,n}$ contain entries $\bm{K}_{ij} = K_\mathcal{X}(x_i,x_j)$, $\bm{L}_{ij} = K_\mathcal{Y}(y_i,y_j)$.
If $x$ and $y$ are dependent then $\widehat{\mathrm{HSIC}}_u$ is asymptotically normally distributed \citep[Theorem 5]{song2012feature}. However, if they are
independent then U-statistic is degenerate and the asymptotic distribution of $\widehat{\mathrm{HSIC}}_u$ deviates from normal 
\citep{gretton2008kernel,serfling2009approximation}. For a given block size $B$, the block estimator
of \cite{zhang2018large} is defined as $\widehat{\mathrm{HSIC}}_{Blo}(\bm{z}) =
\frac{B}{n}\sum_{i=1}^{\frac{n}{B}}\widehat{\mathrm{HSIC}}_u(\{z_j\}_{j=(i-1)B+1}^{iB})$. 
If $\lim_{n,B \rightarrow
\infty} n/B = \infty$, it can be shown that
the block estimator is normally distributed asymptotically even when $x$ is
independent of $y$ \citep[Section 3.2]{zhang2018large}. 

\textbf{Multiscale Bootstrap} A procedure
that calculates ``approximately unbiased'' $p$-values is called multiscale bootstrap proposed by \citet{shimodaira2002approximately,shimodaira2004approximately}.
It was initially proposed for a general statistical problem, called the problem
of regions \citep{efron1998problem}, where we want to compute asymptotically
accurate $p$-values for the null hypothesis $H_0: \mu \in H$ where $H$ is represented by a
region with $H \subseteq \mathbb{R}^d$ (called ``hypothesis region''). \cite{efron1996bootstrap} studied this problem under the normal model $y \sim 
\mathcal{N}(\mu, I)$ and argued that the bootstrap
probabilities $\bs(H):=\prob(y\in H)$ are biased frequentist confidence
measures. Furthermore, they showed that geometric quantities play a crucial role and bias corrected
$p$-values can be produced by using the pivotal quantity $\beta_0(y) - \beta_1 \sim \mathcal{N}(0,1)$
where $\beta_0(y)$ is the signed distance from $y$ to $\partial H$ (i.e., the boundary surface of $H$), and $\beta_1$ is
the mean curvature of $\partial H$.
More specifically, a second-order asymptotically accurate $p$-value is expressed as $p(H|y):=\bar{\Phi}(\beta_0(y)-\beta_1)$, i.e., $\forall \mu \in \partial H,\ \mathbb{P}( p(H|y) < \alpha) = \alpha +\mathcal{O}(n^{-1})$ \citep{efron1996bootstrap, shimodaira2002approximately}.
Note that $\Phi(\cdot)$ is the CDF of the standard normal distribution and $\bar{\Phi}(x):=1-\Phi(x)$. 
However, typically $\beta_0(y)$ and $\beta_1$ are hard to determine due to either the
intractability of the space or the lack of an explicit formulation in the region.
Multiscale bootstrap addresses this problem with additional computation and only requires the regions to be represented by
a function that indicates if $y\in H$ or $y\notin H$.

Let $\mathcal{X}_n= \{x_i\}_{i=1}^n$ be a dataset of sample size $n$ with each element $x_i\in \mathbb{R}^d$. We assume that there is some transformation $f_n$ such that the observed value 
$f_n(\mathcal{X}_n)$  follows a multivariate normal distribution, i.e., $y:=f_n(\mathcal{X}_n) \sim 
\mathcal{N}(\mu,I)$.
\wjsay{transformation for what? If this is the statistic, might be easier to just say so.}
Typically, $f_n$ has a factor $\sqrt{n}$ for scaling the covariance. The main idea of multiscale bootstrap is, instead of $n$ elements, it resamples $n'$ elements from $\mathcal{X}_n$ with replacement to generate $\mathcal{X}^*_{n'}$, then $y^* := f_n(\mathcal{X}^*_{n'}) \sim 
\mathcal{N}(y,\gamma^2I)$ with $\gamma^2 = n/n'$, from which we estimate the desired geometric quantities $\beta_0(y)$ and $\beta_1$
using the scaling law of bootstrap probabilities \citep{shimodaira2002approximately, shimodaira2014higher}.
It can be shown that the bootstrap probability of the region $H$ is expressed as $\bs_{\gamma^2}(H):=\mathbb{P}(y^* \in H)\approx 
\bar{\Phi}(\gamma^{-1}\beta_0(y)+\gamma\beta_1)$. \citet{shimodaira2008testing, shimodaira2014higher}
proposed the normalised bootstrap $z$-value as
$\psi_{\gamma^2}(y|H):=\gamma\bar{\Phi}^{-1}(\bs_{\gamma^2}(H)) \approx \beta_0(y)+\gamma^2\beta_1$ from which $p$-values proposed by \citet{efron1998problem}, namely $p(H|y) = \bar{\Phi}(\beta_0(y)-\beta_1)$, can
be calculated when $\gamma^2=-1$, i.e., we have
$\bar{\Phi}(\psi_{-1}(y|H))=\bar{\Phi}(\beta_0(y)-\beta_1)=p(y|H)$.
 However, it is
impossible to simulate the case where $\gamma^2 = -1$. Multiscale bootstrap tackles this problem by using a number of different sample sizes $n' \in \mathcal{M} \subset \mathbb{N}^{+}$. For each $n'$, we run bootstrap resampling of $\mathcal{X}^*_{n'}$ from $\mathcal{X}_n$ for calculating the normalised bootstrap $z$-value $\psi_{\gamma^2_{n'}}(y|H))$ with $\gamma^2_{n'}:=n/n'$.
The tuple $\{(\gamma^2_{n'}, \psi_{\gamma^2_{n'}}(y|H)) \}_{n' \in \mathcal{M}}$ is used to fit a regression model $\varphi_H(\gamma^2)$ which can then be extrapolated to $\varphi_H(-1)$. The regression model $\varphi_H(\cdot)$ can be used to calculate our $p$-values since
$p(y|H)=\Bar{\Phi}(\varphi_H(-1))$.
 See \citet[Section 6.5]{shimodaira2019selective} for several possibilities of regression models.

\textbf{Selective Multiscale Bootstrap}
Multiscale bootstrap can be extended to the problem of selective inference, where the hypothesis
is random and chosen from the data \citep{shimodaira2019selective,terada2017selective, terada2019selective}. In this problem, there is an additional region $S$ called the ``selective region'' that determines the null hypothesis we are going to test. 
If $y\in S$ then we test $H_0 : y\in H$. However, if $y\not\in S$, we ignore $H$ and no
decision is made. \citet{terada2017selective, terada2019selective} proposed
the following selective $p$-value
$$p(H|y,S) = \frac{\bar{\Phi}(\beta_0(y|H) - \beta_1)}{\bar{\Phi}(\beta_0(y|H) - \beta_1 + \beta_0(y|S))},$$
where $\beta_1$ is the mean curvature of $\partial H$, and
$\beta_0(y|H)$ and $\beta_0(y|S)$ are the signed distances from $y$ to $\partial H$ and $\partial S$, respectively. Under certain assumptions, its null
distribution is uniform  over $(0,1)$~\citep{terada2017selective}, i.e., we have
$$
p(H|y,S)\, |\, y\in S,\, \mu \in \partial H \sim \mathcal{U}(0,1).
$$
The calculation of $p(H|y,S)$ may again be non-trivial.
The difficulty arises from
the calculation of the signed distance $\beta_0(y|S)$ for the selective region $S$ but fortunately, we can apply the same idea
as non-selective multiscale bootstrap as mentioned in previous subsection. In this case, another regression model $\varphi_S(\gamma^2)$ is fitted with the bootstrap probabilities $\bs_{\gamma^2}(S)$ for the region $S$. It can be seen that the
signed distance $\beta_0(y|S)$ can then be obtained by extrapolating the model $\varphi_S(\gamma^2)$ to $\gamma^2=0$.
In other words, the selective $p$-value $p(H|y,S)$ can be calculated using two regression models $\varphi_H(\cdot)$ and $\varphi_S(\cdot)$ as follows,
\begin{equation*}
p(H|y,S) = \frac{\Bar{\Phi}(\varphi_H(-1))}{\Bar{\Phi}(\varphi_H(-1)+\varphi_S(0))}.
\end{equation*}
\begin{algorithm}[]
    \caption{Selective Multiscale Bootstrap}
    \label{msboot}
    \begin{algorithmic}[1] 
        \Procedure{SelectiveMS}{$\mathds{1}_H, \mathds{1}_S, \alpha,\,\bm\mu,\bm\Sigma$}
        \For{$n'\in \mathcal{M}$ }
                \State $\gamma^2_{n'} \gets \frac{n}{n'}$
                \State Sample $\{\bm{y}^*_i\}^B_{i=1} \overset{i.i.d.}{\sim} \mathcal{N}(\bm\mu,\, \gamma^2_{n'}\bm\Sigma)$
                \State $\bs_{\gamma_{n'}^2}(H) \gets \sum^B_{i=1}\mathds{1}_H(\bm{y}^*_i)/B$
                \State $\bs_{\gamma_{n'}^2}(S) \gets \sum^B_{i=1}\mathds{1}_S(\bm{y}^*_i)/B$
            \EndFor
            \State Fit a model $\varphi_H(\gamma^2)$ such that $\varphi_H(\gamma^2)=\gamma\bar{\Phi}^{-1}(\bs_{\gamma^2}(H))$.
            \State Fit a model $\varphi_S(\gamma^2)$ such that $\varphi_S(\gamma^2)=\gamma\bar{\Phi}^{-1}(\bs_{\gamma^2}(S))$.\\
            \Return $\Bar{\Phi}(\varphi_H(-1))/\Bar{\Phi}(\varphi_H(-1)+\varphi_S(0))$
        \EndProcedure
    \end{algorithmic}
\end{algorithm}
Algorithm \ref{msboot} describes selective multiscale bootstrap algorithm where
$\mathds{1}_H(y)$ is an indicator
function: it is $1$ if $y \in H$, otherwise $0$. 
%

%% file: proposal_mmd.tex
\section{SELECTIVE INFERENCE WITH MMD}
\label{sec:si_mmd}
In this section, we propose our first test. We are concerned with the
following problem,

\begin{problem}
Given two distributions $P$ and $Q$ with common support on $\mathcal{X}^d$, we have $n$ i.i.d.\  samples denoted as $\bm{X}_n = [\bm{x}_1, \dots, \bm{x}_n]^\top \in \mathcal{X}^{n\times d}$
with $\bm{x} \sim P$ and similarly for $\bm{Y}_n$ with $\bm{y} \sim Q$.
Our goal
is to find a set of features $\mathcal{S}$ such that for $i \in \mathcal{S}$, the marginal distributions of the $i$-th dimension of $\bm{x}$ and $\bm{y}$ (denoted as $P^{(i)}$ and $Q^{(i)}$) are significantly different,
i.e., $P^{(i)}  \neq Q^{(i)}$.
\end{problem}

The problem and a solution were initially proposed by \citet{yamada2018post}
based on the polyhedral lemma for post-selection inference. Their proposal
mmdInf, which is referred to as PolyMMD in this paper, first selects a set of
$k$ features $\mathcal{S}_k$ using MMD
and then tests if each of the selected features' marginal distributions are
different. The latter part is performed by conditioning on selecting the
whole set $\mathcal{S}_k$. This form of conditioning can be written
equivalently as a set of linear constraints \citep[section
3.1]{yamada2018post} and as a result, it is possible to employ the polyhedral
lemma and obtain a truncated normal as their asymptotic null distribution. However, we
can relax the conditioning further. Notice that the goal is to test each
feature $i \in \mathcal{S}_k$ separately.
Thus, given a significance level $\alpha$, it is sufficient to require
type-I error to be no larger than $\alpha$, conditioned only on $i \in
\mathcal{S}_k$, rather than on the full set $\mathcal{S}_k$.
Following \cite{liu2018more}, we call this event the minimal conditioning set.
While the selection event of 
PolyMMD can be written as a single polyhedron, the selection event  $i \in
\mathcal{S}_k$ is more complicated.

In this section, we propose MultiMMD a more powerful variant of PolyMMD by conditioning on the minimal conditioning set. We show how the statistical test can be
performed using multiscale bootstrap. Although for the remainder of
the section we focus on the incomplete estimator $\widehat{\mathrm{MMD}}^2_{Inc}$, a similar procedure can be applied to the block estimator $\widehat{\mathrm{MMD}}^2_{Blo}$ and the linear time estimator $\widehat{\mathrm{MMD}}^2_{l}$.
%
%
\subsection{Proposal: \textrm{MultiMMD}}
From the index set of all features $\mathcal{I}$, \textrm{MultiMMD} finds a subset of $k$ features, denoted by $\mathcal{S}_k \subseteq \mathcal{I}$, that differentiates samples from $P$ and $Q$. The $k$ features are selected as the $k$ dimensions with the highest scores measured by an estimator of \textrm{MMD}.
More precisely, we have $\mathcal{S}_k = \mathcal{S}_{k-1} \cup \{ \arg\max_{i \in \mathcal{I} \setminus \mathcal{S}_{k-1}}
\widehat{\mathrm{MMD}}^2_{Inc} (\bm{X}_n^{(i)}, \bm{Y}_n^{(i)} )\}$ where $\bm{X}_n^{(i)}=[\bm{x}^{(i)}_1,\dots,\bm{x}^{(i)}_n]^\top$ and $\bm{x}^{(i)}$ is the $i$-th dimension of the random variable $\bm{x}$ (and similarly for $\bm{Y}_n^{(i)}$) and $\mathcal{S}_0 = \emptyset$.

The selection procedure mentioned above for MultiMMD is the same as PolyMMD, but the statistical test we perform is different. For each selected feature $i \in \mathcal{S}_k$, 
the hypothesis test we execute is
\begin{align*}
    H_{0,i}: \mathrm{MMD}^2(P^{(i)}, Q^{(i)}) = 0\ |\ i \in \mathcal{S}_k \text{ is selected}, \\
    H_{1,i}: \mathrm{MMD}^2(P^{(i)}, Q^{(i)}) > 0\ |\ i \in \mathcal{S}_k \text{ is selected}.
\end{align*}
In contrast to PolyMMD, when testing $H_{0,i}$ for
some $i \in \mathcal{S}_k$, all 
the other selected variables in $\mathcal{S}_k\setminus \{i\}$
are not considered for conditioning.
The justification of why tests that condition on the minimal
conditioning set is more powerful can be found in \citet{fithian2014optimal}. The main idea is that
the monotonicity of the selective type I error, defined as $\mathbb{P}(\mathrm{Reject\,} H_{0,i} \le \alpha\, |\,H_{0,i}, i \in \mathcal{S}_k)$,  suggests that
we will lose power when we move from coarse selection variables to finer
selection variables \citep[Proposition
3]{fithian2014optimal}.

Multiscale bootstrap's flexibility in representing hypothesis and selective regions with indicator functions makes it a suitable candidate to calculate $p$-values.
It requires us to define a set $\mathcal{M}$ with each of its members $n' \in \mathcal{M}$ specifying the number of elements to be resampled from  $\bm{X}_n$ and $\bm{Y}_n$ (denoted as $\bm{X}_{n'}$ and $\bm{Y}_{n'})$.  
We  generate bootstrap replicates of statistic
$\hat{\bm{T}}_{n}(\bm{X}_{n'},
\bm{Y}_{n'}):=\sqrt{l_n}[\widehat{\mathrm{MMD}}^2_{Inc} (\bm{X}_{n'}^{(1)},
\bm{Y}_{n'}^{(1)}),\dots, \widehat{\mathrm{MMD}}^2_{Inc} (\bm{X}_{n'}^{(d)},
\bm{Y}_{n'}^{(d)})]^\top$ where $l_n = |\mathcal{D}_n|$ is the denominator used
for $\widehat{\mathrm{MMD}}^2_{Inc}  (\bm{X}_{n}^{(i)}, \bm{Y}_{n}^{(i)})$. For each $n' \in \mathcal{M}$, the statistic is
computed $B$ times and the
bootstrap probability $\bs_{\gamma^2}(\cdot)$ of the hypothesis region $H$ (and selective region $S$) is the average 
number of the $B$ samples that falls within $H$ (and $S$). As a result, we require a 
sampler for
$\hat{\bm{T}}_{n}(\bm{X}_{n'}, \bm{Y}_{n'})$ for all $n' \in \mathcal{M}$ and
a function that describes whether the statistic falls within the regions $H$ and $S$.
Finally, two linear regression models are fitted: one for $H$ and one for $S$
denoted as $\varphi_H(\cdot)$ and $\varphi_S(\cdot)$ respectively. Assuming that the boundary surfaces can be represented by a polynomial of degree 3, then the existing theory recommends a linear model \citep[Section 5.4]{shimodaira2008testing}.
The model's predictor variable is the ratio $\gamma^2_{n'} := \frac{n}{n'}$ and its response variable
is $\gamma_{n'}\bar{\Phi}^{-1}(\bs_{\gamma^2_{n'}}(\cdot))=\beta_0+\gamma^2_{n'}\beta_1$ where $\beta_0$ is the signed distance from our statistic to the boundary of the region and $\beta_1$ the mean curvature at the boundary.

 \begin{algorithm}[]
    \caption{MultiMMD($\bm{X}_n,\, \bm{Y}_n, k,\, \mathcal{M}$): Selective $p$-values for the null hypothesis $H_{0,i}: \mathrm{MMD}^2(P^{(i)}, Q^{(i)}) = 0\ |\ i \in \mathcal{S}_k \text{ is selected}$.}
    \label{alg:multimmd}
    \begin{algorithmic}[1] 
            \State $\hat{\bm{T}}_{n}(\bm{X}_n, \bm{Y}_n), \hat{\bm\Sigma} \gets $EstimateParam($\bm{X}_n$, $\bm{Y}_n$)
            \State $\mathcal{S}_k \gets $ the indexes
        of $k$ largest values of $\{\widehat{\mathrm{MMD}}^2(\bm{X}_n^{(i)}, \bm{Y}_n^{(i)})\}_{i\in\mathcal{I}}$
            \For{$i \in \mathcal{S}_k$}
                \For{$n' \in \mathcal{M}$}
                    \State $\gamma^2_{n'} \gets \frac{n}{n'}$
                    \State Sample $\{\bm{y}^*_i\}^B_{i=1} \overset{i.i.d.}{\sim} \mathcal{N}(\hat{\bm{T}}_{n}(\bm{X}_n, \bm{Y}_n),\, \gamma^2_{n'}\hat{\bm\Sigma})$
                    \State $\bs_{\gamma_{n'}^2}(S) \gets \sum^B_{i=1}\mathds{1}^{(i)}_S(\bm{y}^*_i)/B$
                \EndFor
                \State Fit a linear model $\varphi_S(\gamma^2)$ such that $\varphi_S(\gamma^2)=\gamma\bar{\Phi}^{-1}(\bs_{\gamma^2}(S))$.
                \State $\hat{\beta}_0^{(i)} \gets \hat{\sigma}_i^{-1}\sqrt{l_n}\widehat{\mathrm{MMD}}^2_{Inc}(\bm{X}_n^{(i)}, \bm{Y}_n^{(i)})$
                \State $p_i \gets \Bar{\Phi}(\hat{\beta}_0^{(i)})/\Bar{\Phi}(\hat{\beta}_0^{(i)}+\varphi_S(0))$
            \EndFor
            \State \Return $\{p_i\}_{i=0}^k$ and $\mathcal{S}_k$
    \end{algorithmic}
\end{algorithm}
We begin by describing how to obtain samples of $\hat{\bm{T}}_{n}(\bm{X}_{n'}, \bm{Y}_{n'})$ for all $n' \in \mathcal{M}$. Suppose that the bootstrap resamples can be represented using the distribution
$\mathcal{N}(\hat{\bm{T}}_{n}(\bm{X}_n, \bm{Y}_n),\hat{\bm\Sigma})$ where $\hat{\bm\Sigma}$ be the sample
covariance of $\widehat{\mathrm{MMD}}^2_{Inc}$, i.e., $\hat{\bm\Sigma}:=\frac{1}{l_n-1}\sum_{(i,j) \in \mathcal{D}_n} [\bm{h}(\bm{z}_i,\bm{z}_j) - \overline{\bm{h}}][\bm{h}(\bm{z}_i,\bm{z}_j) - \overline{\bm{h}}]^\top$ where (recall that $h(\cdot,\cdot)$ is the U-statistic kernel) $\bm{h}(\bm{z}_i,\bm{z}_j) := [h(\bm{z}^{(1)}_i,\bm{z}^{(1)}_j), \ldots, h(\bm{z}^{(d)}_i,\bm{z}^{(d)}_j)]^\top \in \mathbb{R}^d$ and $\overline{\bm{h}} := \frac{1}{l_n}\sum_{(i,j) \in \mathcal{D}_n} \bm{h}(\bm{z}_i,\bm{z}_j)$. The choice of normal distribution is justified as $\hat{\bm{T}}_{n}(\bm{X}_n, \bm{Y}_n)$ tends to be normally distributed as $n \rightarrow \infty$ \citep[Theorem 5]{yamada2018post}.
In order to replicate samples of $\hat{\bm{T}}_n(\bm{X}_{n'},\bm{Y}_{n'})$, notice that
its asymptotic distribution is $\hat{\bm{T}}_n(\bm{X}_{n'},\bm{Y}_{n'}) \sim
\mathcal{N}(\hat{\bm{T}}_{n}(\bm{X}_n, \bm{Y}_n), \frac{l_n}{l_{n'}}\hat{\bm\Sigma})$.
For each $n'$, instead of resampling $n'$ elements from $\bm{X}_n$ and $\bm{Y}_n$ for calculating $\hat{\bm{T}}_n(\bm{X}_{n'},\bm{Y}_{n'})$ $B$ times,  we generated $B$ replicates directly from $\mathcal{N}(\hat{\bm{T}}_{n}(\bm{X}_n, \bm{Y}_n), \frac{l_n}{l_{n'}}\hat{\bm\Sigma})$ which is then used to calculate bootstrap probabilities. 
The former is an $\mathcal{O}(n'B)$ process while the latter is $\mathcal{O}(B)$.
In practice, the $B$ replicates are sampled from $\mathcal{N}(\hat{\bm{T}}_{n}(\bm{X}_n, \bm{Y}_n), \frac{n}{n'}\hat{\bm\Sigma})$ instead because we let $l_n$ and $l_{n'}$ be the typical choice of $l_{n} = rn$ and $l_{n'} = rn'$
where $r$ is fixed apriori with $0 < r < \infty$.
The choice of $r$ affects the distribution of $\hat{\bm{T}}_{n}(\cdot, \cdot)$ \citep[section 4]{yamada2018post}. When $r$ is high, its asymptotic distribution tends towards its complete counterpart (i.e., infinite
sums of weighted chi-squared variables). But when $r$ is small, 
it is normally distributed.

For each $n'$, the $B$ replicates are used to calculate bootstrap probabilities
for both the hypothesis region $\bs_{\frac{n}{n'}}(H)$ and the
selective region $\bs_{\frac{n}{n'}}(S)$.
Note that the hypothesis $H_{0,i}$ can be written as the region $H = \{y\in
\mathbb{R}^d : y^{(i)} \le 0\}$ which has a flat boundary.  This means that the
curvature is $\beta_1 = 0$ and multiscale
bootstrap is not needed for the hypothesis region $H$. In fact, the signed distance for testing
$i\in \mathcal{S}_k$
is $\hat{\beta}^{(i)}_0 =
\hat{\sigma}^{-1}_{i}\sqrt{l_n}\widehat{\mathrm{MMD}}^2_{Inc}(\bm{X}_n^{(i)},
\bm{Y}_n^{(i)})$ where $\hat{\sigma}^{2}_{i}$ is the $i$th diagonal element of
$\hat{\bm\Sigma}$. We have
$\varphi_H(\gamma^2)=\hat{\beta}^{(i)}_0$ as a constant function.

However, it
is not as easy for the selective region $S$ which requires the application of
multiscale bootstrap. $S$ is represented by an indicator function 
$
    \mathds{1}_S^{(i)}(\bm{y}^*) = 
    \begin{cases}
    1 & \text{if } i \in \mathcal{S}^*_k \\
    0 & \text{if } i \notin \mathcal{S}^*_k
    \end{cases}
    $
where $\mathcal{S}^*_k$ is the selected set of
$k$ features where our selection algorithm is applied to $\bm{y}^*$. Let $\gamma_{n'}^2=\frac{n}{n'}$ then, the bootstrap probability is given
by $\bs_{\gamma_{n'}^2}(S) = \sum^B_{i=1}\mathds{1}^{(i)}_S(\bm{y}^*_i)/B$ where $\{\bm{y}^*_i\}^B_{i=1} \overset{i.i.d.}{\sim} \mathcal{N}(\hat{\bm{T}}_{n}(\bm{X}_n, \bm{Y}_n),\, \gamma^2_{n'}\hat{\bm\Sigma})$.
 For a given $\mathcal{M}=\{n'\}$, we have $|\mathcal{M}|$ pairs of predictor and response $\{(\gamma^2_{n'},\gamma_{n'}\bar{\Phi}^{-1}(\bs_{\gamma^2_{n'}}(S)))\}_{n'\in\mathcal{M}}$ that is used to fit a linear model $\varphi_S(\gamma^2)$. We define $\mathcal{M}$ to be the set of numbers equally spaced between $0.5n$ to $2n$ in log space with $|\mathcal{M}| = 10$. The function $\varphi_S(\gamma^2)$ can be used to extrapolate to $\gamma^2=0$ to obtain the signed distance from our statistic to the boundary of $S$. 
 Then, our selective $p$-value  for feature
 $i\in \mathcal{S}_k$ is given by 
 $$p_i = \Bar{\Phi}(\hat{\beta}_0^{(i)})/\Bar{\Phi}(\hat{\beta}_0^{(i)}+\varphi_S(0)).$$
 We reject $H_{0,i}$ if $p_i < \alpha$. The algorithm is described in Algorithm \ref{alg:multimmd}.

%% file: proposal_hsic.tex
\section{SELECTIVE INFERENCE WITH HSIC}
\label{sec:si_hsic}

We consider the problem studied in \citet{yamada2016post}

\begin{problem}
\label{prob:feat_sel}
Given $n$ samples from the joint distribution $\{(\bm{x}_i,\bm{y}_i)\}^n_{i=1}\overset{i.i.d.}{\sim}P_{\bm{x}\bm{y}}$ on the domain $\mathcal{X}^d \times \mathcal{Y}$,
our goal is to find a subset $\mathcal{S}$ of features of $\bm{x}$ such that 
for each $i \in \mathcal{S}$ there is statistically significant
dependency between the feature $\bm{x}^{(i)}$ and response $\bm{y}$.
\end{problem}

The goal is to decide if there is some dependence between the marginal distribution $\bm{x}^{(i)}$ and the response $\bm{y}$. Whereas in MultiMMD, it compares if the difference between the marginal distributions of $\bm{x}^{(i)}$ and
$\bm{y}^{(i)}$ is zero, i.e., if
$\mathrm{MMD}^2(\bm{x}^{(i)},\bm{y}^{(i)})=0$. In \citet{yamada2016post}, a
solution was proposed for Problem \ref{prob:feat_sel} using HSIC to measure the dependency between the two between $\bm{x}^{(i)}$ and $\bm{y}$ which we call ``PolyHSIC'' (previously called hsicInf). It begins with first selecting $k$
features with the highest HSIC scores and then a test is performed for each
feature. But the
conditioning is not minimal and suffers from a loss in power. 

In this section, we propose a new estimator based on the incomplete
U-statistic estimator for HSIC and analyse its asymptotic distribution. We
then extend PolyHSIC to
``MultiHSIC''. The new proposal conditions on the minimal conditioning
set which is made possible with the multiscale bootstrap. 
Although our procedure allows
the use of the HSIC block estimator \citep{zhang2018large}, 
we observe that the convergence of the estimator to the target normal
distribution highly depends on the block size and in turn the number of blocks, which
can be challenging to set correctly. As a result, the false rejection rate of
the test can be difficult to control. See Appendix \ref{sec:appen_bias} for a
numerical simulation that illustrates this problem, and \citet[page
5]{zaremba2013b} for a discussion on a similar issue in the block estimator for MMD.

\subsection{Incomplete HSIC}
\label{sec:inchsic}

We propose
an estimator of HSIC that behaves as desired with the type I error at size $\alpha$ (unlike the block estimator) when used in conjunction with multiscale bootstrap. It based on the 
incomplete U-statistic estimator \citep{blom1976some, janson1984asymptotic} defined as
$$
\widehat{\mathrm{HSIC}}_{Inc}(\bm{z})= \frac{1}{l}\sum_{(i,j,q,r) 
\in \mathcal{D}} h(i,j, q,r)
$$
where $\bm{z} = [z_1,\dots,z_n]$ of $n$ i.i.d.\ draws from $z:=(x,y)\sim P_{xy}$,
$l=|\mathcal{D}|$, and $\mathcal{D}$ is the \textit{design} of the matrix
and for $\widehat{\mathrm{HSIC}}_{Inc}$ it is constructed randomly by sampling $l$ terms with replacement from $\bm{i}^n_4$. 

The asymptotic distribution of  $\mathrm{\widehat{HSIC}}_{Blo}$ is normal in both cases when $y$ and $x$ are independent and dependent \citep[section 3.2]{zhang2018large}.
As shown in Corollary \ref{thm:asympt_inchsic2}, it also follows that
$\widehat{\mathrm{HSIC}}_{\text{Inc}}$ is asymptotically normal regardless of
the presence of the dependency between $x$ and $y$ (in Appendix
\ref{sec:append_emp_dist}, 
we empirically validate this claim).

\begin{restatable}[Asymptotic Distribution of $\widehat{\mathrm{HSIC}}_{\text{Inc}}$]{corollary}{asymptinchsico}
\label{thm:asympt_inchsic2}
Assume that $\lim_{n,l \rightarrow \infty} n^{-2}l = 0$ and $0 < \lim_{n,l \rightarrow \infty} n^{-1}l = \lambda < \infty$,
\begin{itemize}
    \item If $X \independent Y$, then $l^{\frac{1}{2}}\widehat{\mathrm{HSIC}}_{\text{Inc}}(\bm{z}) \overset{d}{\rightarrow} \mathcal{N}(0,\sigma^2)$,
    \item If $X \notindependent Y$, then $l^{\frac{1}{2}}(\widehat{\mathrm{HSIC}}_{\text{Inc}}(\bm{z}) -\mathrm{HSIC}(P_{xy})) \overset{d}{\rightarrow} \mathcal{N}(0,\lambda\sigma_u^2+\sigma^2)$,
\end{itemize}
where $\sigma^2=\mathrm{Var}[h(i,j,q,r)]$ and $\sigma_u^2$ is the variance of the 
complete U-statistic counterpart, see \citet[Theorem 5]{song2012feature}.
\end{restatable}

A measure of its performance is asymptotic relative efficiency ($\mathrm{ARE}$) \citep{lee2019u}
of the incomplete estimator with respect its complete counterpart
\begin{align*}
\mathrm{ARE} & =\lim_{n\rightarrow\infty}\mathrm{Var}(\widehat{\mathrm{HSIC}}_{Inc})/\mathrm{Var}(\widehat{\mathrm{HSIC}}_{u})\\
 & =\lim_{n\rightarrow\infty}\frac{(\lambda\sigma_{u}^{2}+\sigma^{2})/l}{\sigma_{u}^{2}/n}
 =1+\frac{\sigma^{2}}{r\sigma_{u}^{2}},
\end{align*}
where a common choice of $l$ is chosen to be $rn$ for some $r>0$. This means that for large $r$ the incomplete estimator is
asymptotically efficient and dependent on the ratio
$\frac{\sigma^2}{\sigma_u^2}$ i.e., if we set $r$ to be big,
we have $\mathrm{ARE} \approx 1$. A similar analysis can be
performed for the incomplete estimator for the MMD which suggests we should
take $r$ to be very high but it would violate our assumption that $\lim_{n\rightarrow \infty} n^{-1}l < \infty$ 
so the
estimator will deviate from normal.

\subsection{Proposal: MultiHSIC}

In this section, we outline \textrm{MultiHSIC} as a more powerful method for
variable selection by considering
the minimal conditioning set. The algorithm begins by selecting $k$ features with the $k$ largest 
$\widehat{\mathrm{HSIC}}$ scores and then performing a test after selection. We have 
\begin{align*}
    H_{0,i}: \textrm{HSIC}(P_{\bm{x}^{(i)}\bm{y}}) = 0\ |\  i \in \mathcal{S}_k \text{ is selected}, \\
    H_{1,i}: \textrm{HSIC}(P_{\bm{x}^{(i)}\bm{y}}) > 0\ |\ i \in \mathcal{S}_k \text{ is selected},
\end{align*}
where $P_{\bm{x}^{(i)}\bm{y}}$ is defined as the joint distribution between
$\bm{x}^{(i)}$ and $\bm{y}$. Define $\bm{Z}^{(i)} = [\bm{z}^{(i)}_1,\dots,
\bm{z}^{(i)}_n]$ and $\bm{z}^{(i)} := (x^{(i)}, y)$.
Let $\mathcal{S}_k = \mathcal{S}_{k-1} \cup \{ \arg\max_{i \in \mathcal{D}
\setminus \mathcal{S}_{k-1}} \widehat{\mathrm{HSIC}}_{Inc}(\bm{Z}^{(i)})\}$ and $\mathcal{S}_0 = \emptyset$.
%

%
\textbf{Related Work:} \cite{slim2019kernelpsi} proposed a class of kernel
based statistics for selecting variables which can later be used for
hypothesis testing. Their selective inference algorithm includes a sampler
for simulating the null hypothesis. While we have focused on feature
selection algorithms that condition on selection events to solve Problem
\ref{prob:feat_sel}, there is another branch of selective inference
procedures based on the knockoff filter \citep{barber2019knockoff} which has
been extended to the high dimensional setting \citep{candes2018panning}.
Their framework is similar in the sense that they too first select promising
features and then provide selective guarantee on the inference made on the
selected variables. However, their guarantees are based on generating
convincing
knock-offs variables such that the joint distribution is invariant between swaps of the variables and its knockoffs, which can be hard
\citep{lu2018deeppink, romano2019deep}. These proposals control false discovery rate \citep{benjamini1995controlling}.

%% file: experiment.tex
\section{EXPERIMENTS}
\label{sec:experiments}
\setlength{\belowcaptionskip}{-5pt}
In this section, we demonstrate our proposed method for both toy and real world
datasets. The performance of our algorithm is measured by true positive rate
(TPR) and false positive rate (FPR) which can be thought of as power and type-I
error. TPR is defined to be the portion of true selected features that are correctly declared as such and FPR quantifies the portion of selected false features that are declared as incorrectly significant (see definitions in Section \ref{sec:tpr_fpr}). 
It is desirable to have high true positive rate and false positive rate to be
controlled at $\alpha$ (it is not desirable for this to be below $\alpha$ or
above $\alpha$) since the threshold is chosen to be such that the type-I error is size $\alpha$. Unless specified otherwise, we use the Gaussian kernel with its bandwidth chosen with the median heuristic.

Our first experiment considers several synthetic problems to evaluate our proposal and verify that our test controls FPR at nominal levels. For MMD, we use the mean shift problem varying both $n$ and $d$ and, for HSIC, we consider the logistic problem. Then, we proceed to using several real world data-sets which have been augmented with artificial and independent features. We consider the original preprocessed features as ``true'' features which allows us to calculate TPR and FPR. For MMD, we split the data-set into two sets for two different classes and the goal here is to ``rediscover'' 
the original features (with a minimal number of artificially added and uninformative features). And for HSIC,
the data-set is split into the predictor 
variables (with some fakes) and the response variable, the goal here is to find the original predictors. For our final experiment, we consider the problem of anomalous dataset detection where $d$ is small (and so $k$ is small too). In this scenario, our algorithm only has incremental increase in power. Additional experiments can be found in the Appendix \ref{sec:appen_exp}. Code for reproducing our results is available online: \url{https://github.com/jenninglim/multiscale-features}.
\subsection{Toy Problems}
\begin{figure}
\centering
\includegraphics[width=80mm]{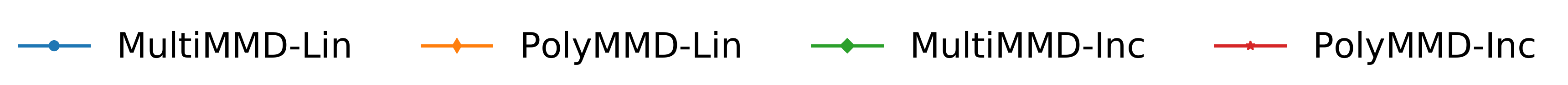}
\subfloat{\includegraphics[width=40mm]{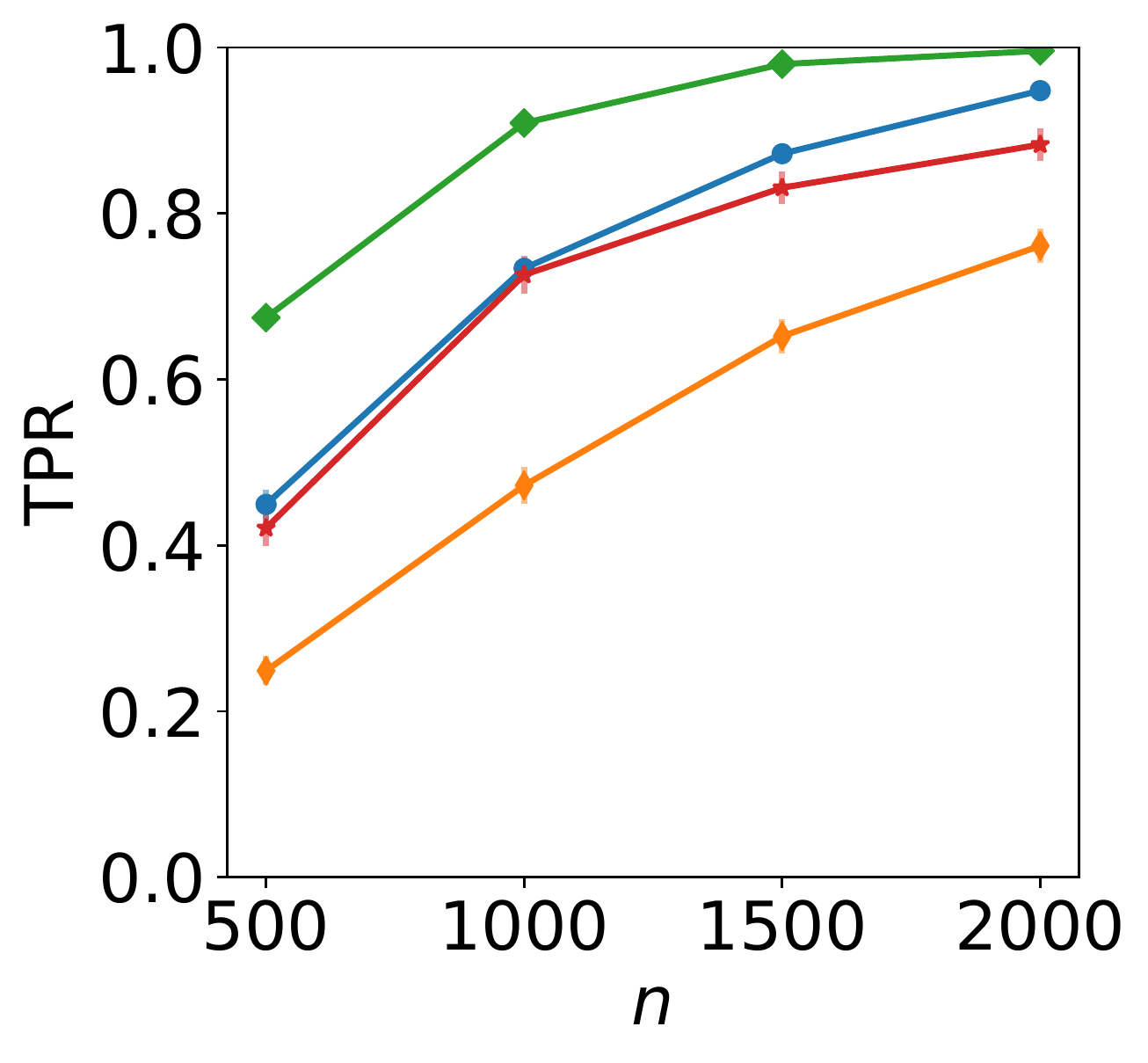}}
\subfloat{\includegraphics[width=40mm]{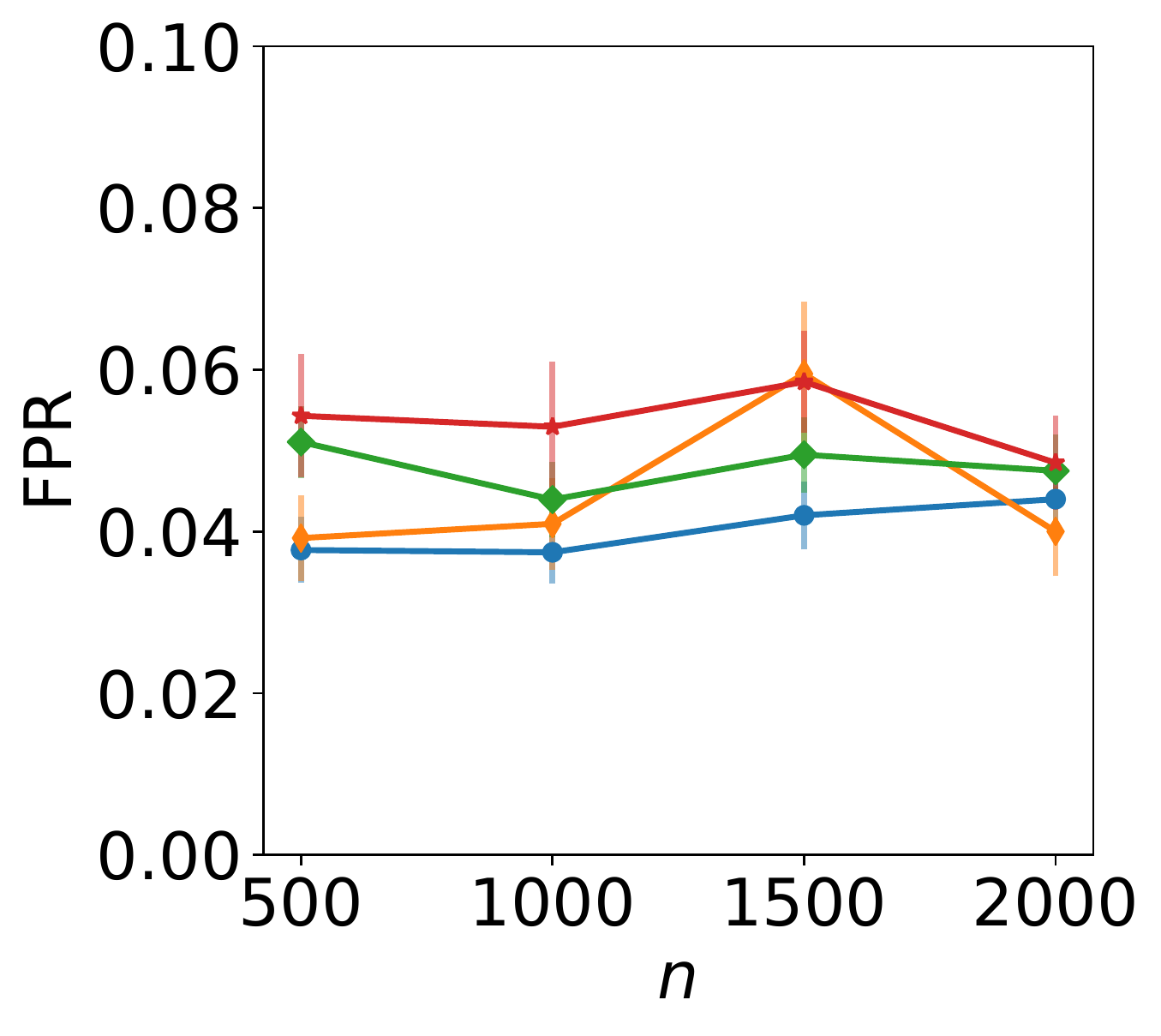}}
\caption{Mean shift experiment as $n$ increases. Results are shown for $\mathrm{\widehat{MMD}}^2_l$ and $\mathrm{\widehat{MMD}}^2_{Inc}$.}
\label{fig:ex1}
\end{figure}
\begin{figure}
\centering
\includegraphics[width=80mm]{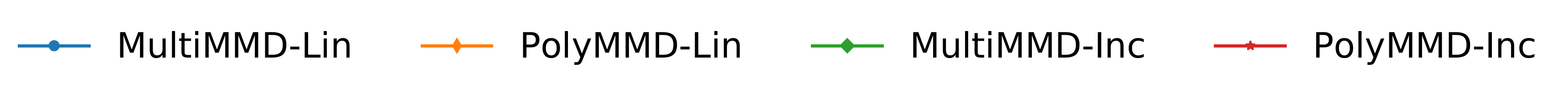}
\subfloat{\includegraphics[width=40mm]{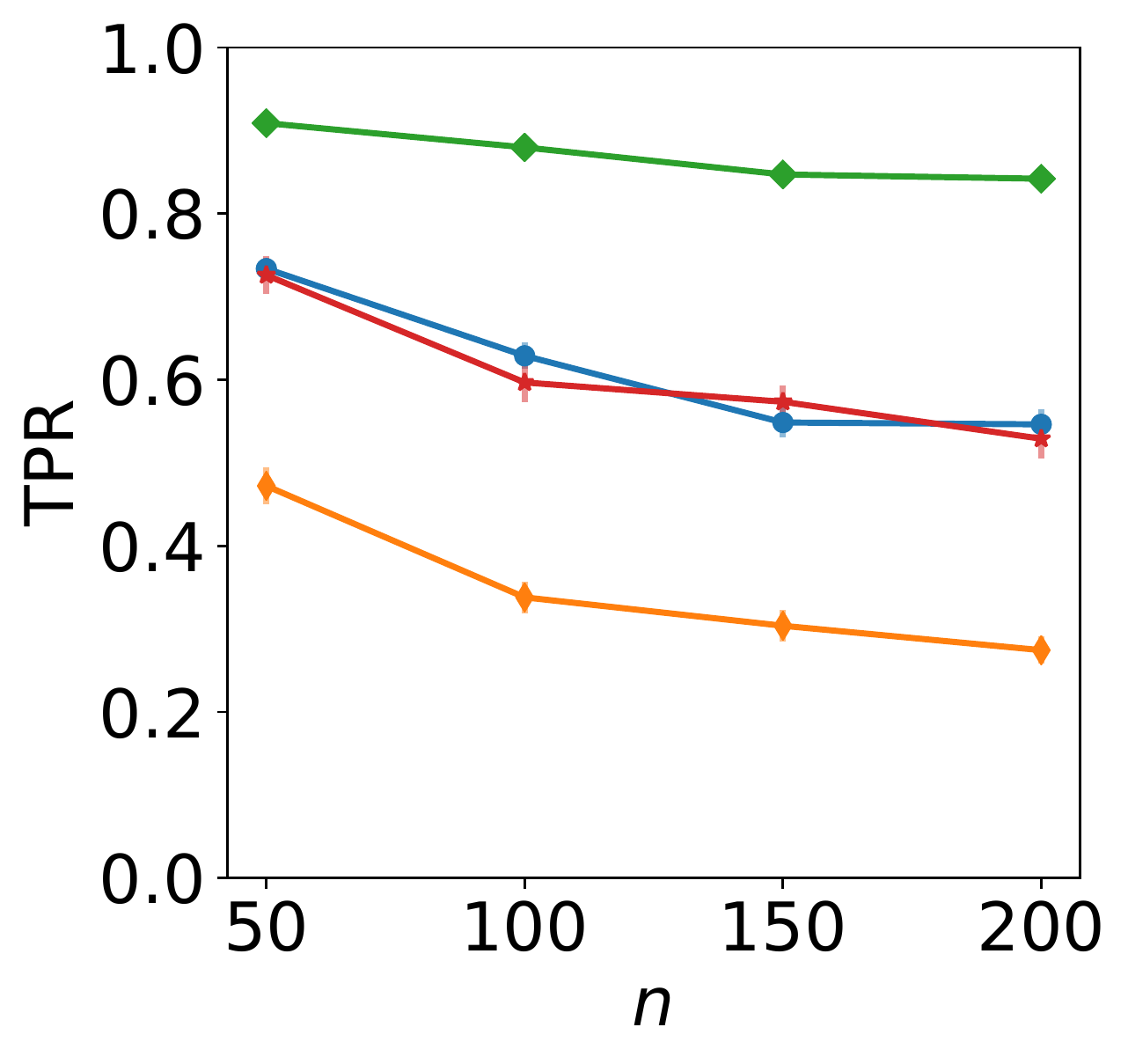}}
\subfloat{\includegraphics[width=40mm]{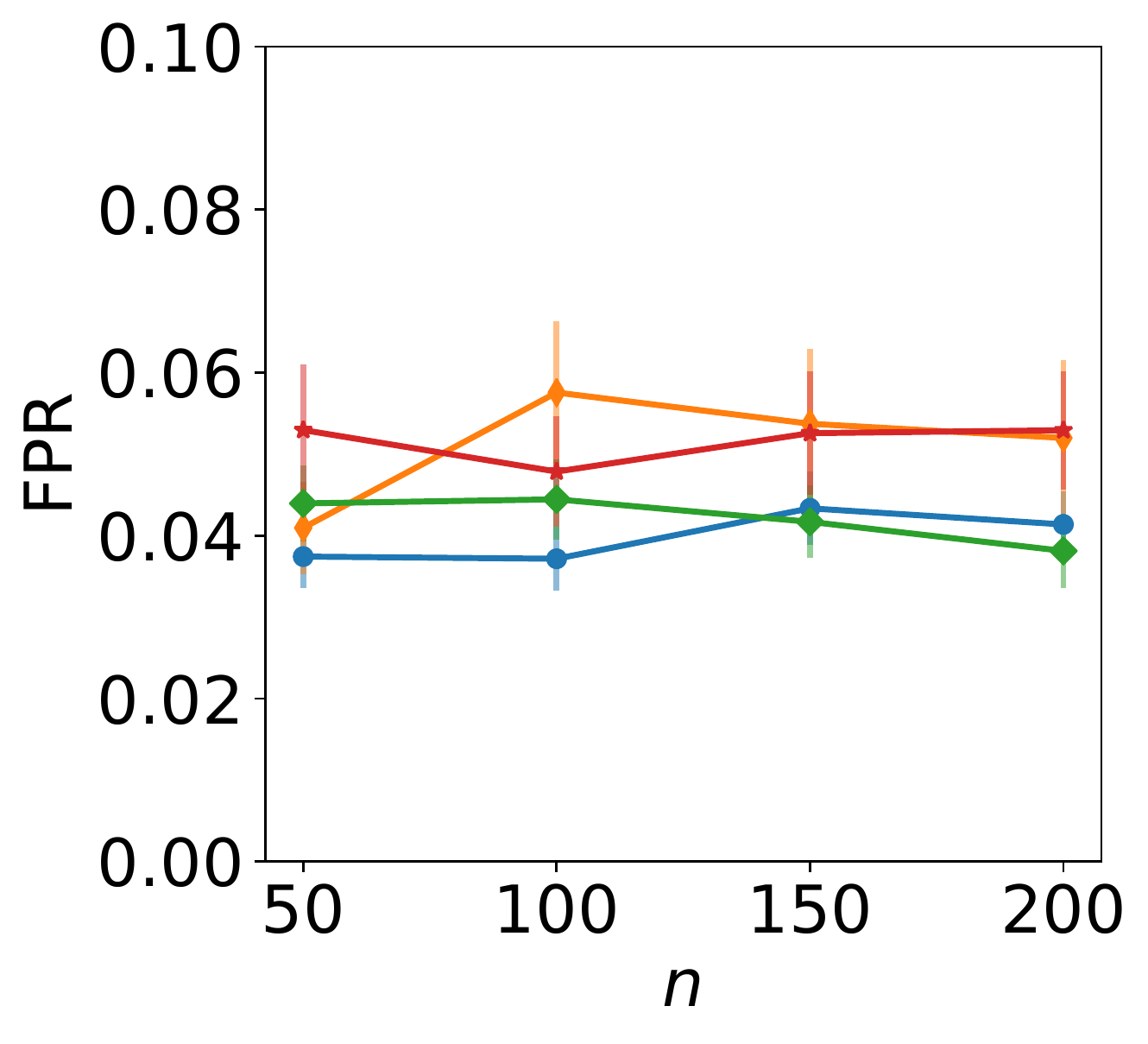}}
\caption{Mean shift experiment as $d$ increases. Results are shown for $\mathrm{\widehat{MMD}}^2_l$ and $\mathrm{\widehat{MMD}}^2_{Inc}$.}
\label{fig:ex2}
\end{figure}
The aim of these synthetic experiment is to evaluate our proposals, \textrm{MultiMMD} and \textrm{MultiHSIC},
against previously proposed methods and empirically verify the theoretical guarantees. The TPR and FPR are averaged over $100$ trials, $k =30$ and $\alpha=0.05$. We consider the three scenarios.

\textbf{MMD: Mean Shift with varying $n$} $(d=50)$. We are given $n$ samples from $P= \mathcal{N}(\bm{0}, \bm{I})$ and $Q= \mathcal{N}(\bm{\mu}, \bm{I})$ where $\mu = [ \bm{0.5}_{10},\  \bm{0}_{40}]^\top \in \mathbb{R}^{50}$. For the first ten rows the alternative holds while for the rest the test should not reject the null hypothesis. This problem was studied in \citet{yamada2018post} and the results are shown in Figure \ref{fig:ex1}.

\textbf{MMD: Mean Shift with varying $d$} $(n=1000)$. The samples are drawn from $P= \mathcal{N}(\bm{0}, \bm{I})$ and $Q= \mathcal{N}(\bm{\mu}, \bm{I})$ where $\mu = [ \bm{0.5}_{10},\  \bm{0}_{d-10}]^\top \in \mathbb{R}^{d}$. The alternative holds only for the first ten rows. The results are shown in Figure \ref{fig:ex2}.

\textbf{HSIC: Logistic problem with varying $n$} $(d=50)$. We consider the feature selection toy experiment studied in \citet{jordon2018knockoffgan, candes2016panning}. We have $\bm{x} = [x_1,...,x_n]$ 
is 
$n$ i.i.d.\ draws from 
$50$-dimensional $\mathcal{N}(\bm{0},\bm{I})$ and $\bm{y} = [y_1,...,y_n]$ with
$y_j \sim \mathrm{Bernoulli}(\mathrm{Logistic}(\sum_{i=1}^{10} x^{(i)}_j))$ where $\mathrm{Logistic}(x) = \frac{\exp (x)}{1+\exp (x) }$. Notice that $\bm{y}$ is dependent only on the first $10$
dimensions of $\bm{x}$ and thus it is desirable to only reject the null hypothesis for these first $10$ features. For the block estimator, we set the block size to $5$; and for the incomplete estimator, we set $r=1$. The results are shown in Figure \ref{fig:ex3}.
\begin{figure}
\centering
\includegraphics[width=80mm]{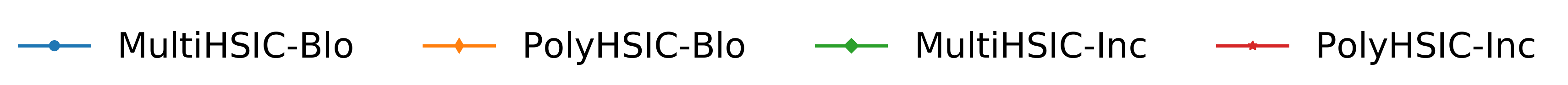}
\subfloat{\includegraphics[width=40mm]{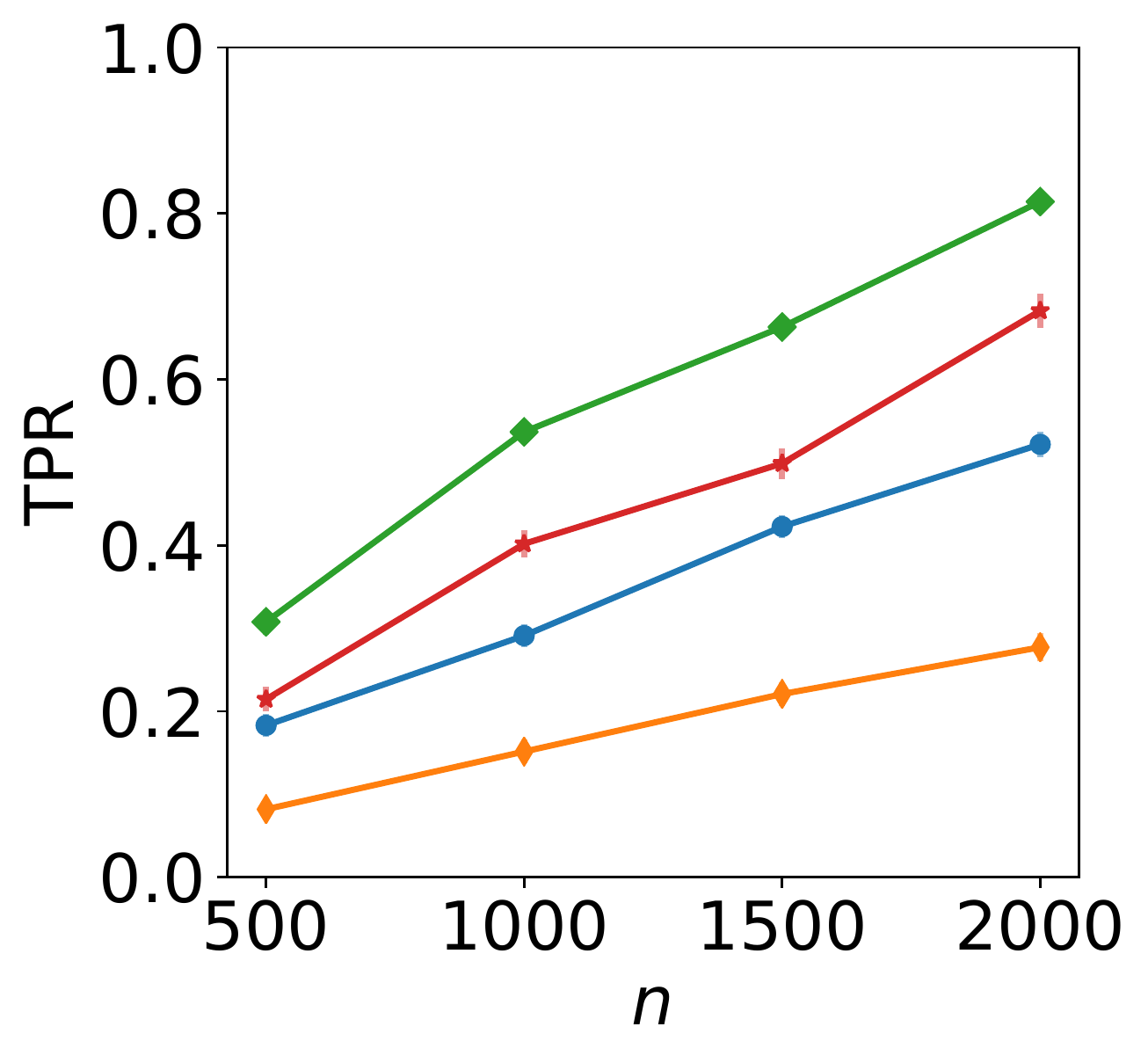}}
\subfloat{\includegraphics[width=40mm]{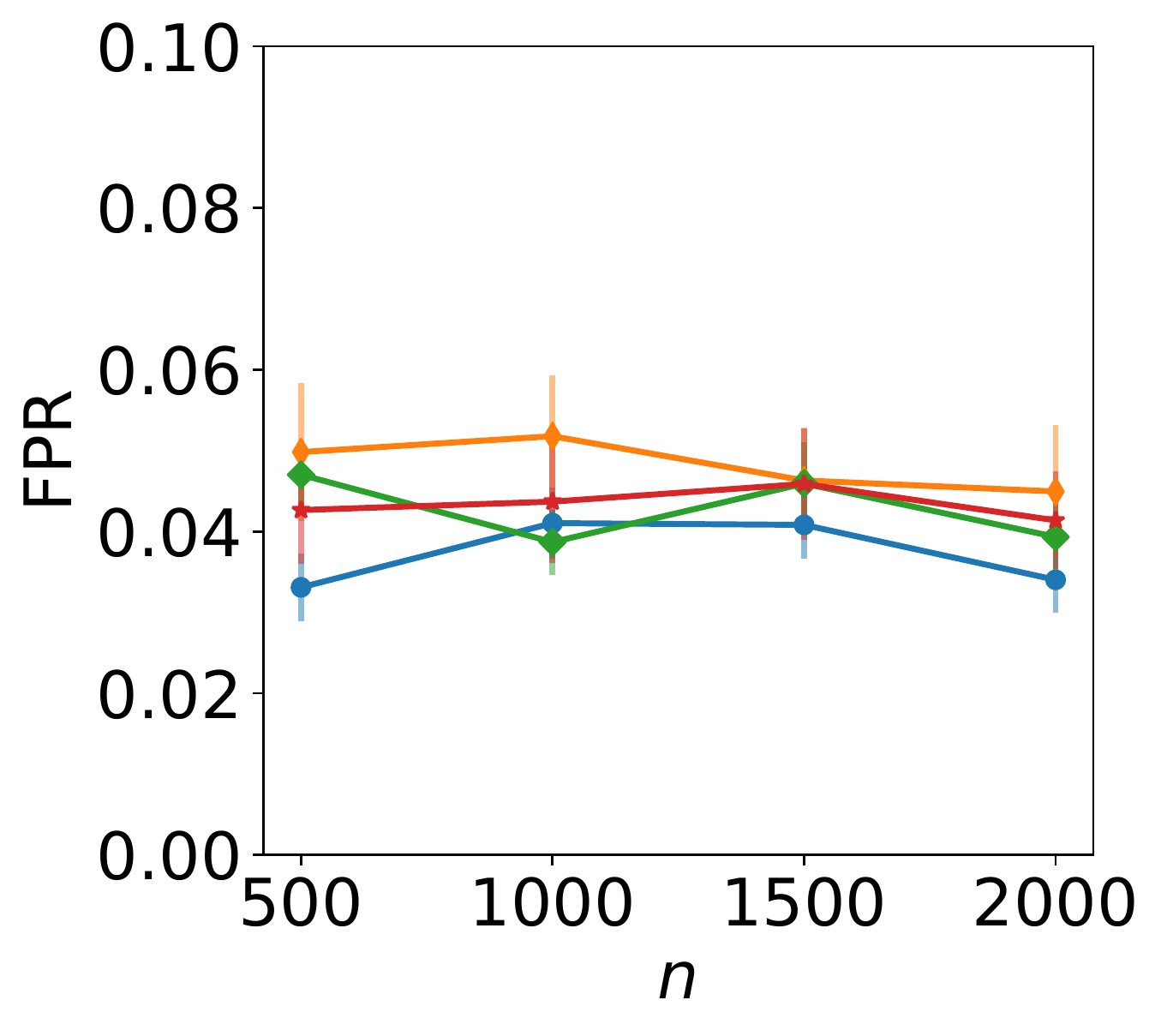}}
\caption{Logistic experiment as $n$ increases. Results are shown for $\mathrm{\widehat{HSIC}}_{Blo}$ and $\mathrm{\widehat{HSIC}}_{Inc}$.}
\label{fig:ex3}
\end{figure}
\subsection{Benchmarks}
\label{sec:mmd_benchmark}
We apply MultiMMD and PolyMMD for selecting features that significantly distinguishes 
two samples. Since TPR and FPR requires the knowledge of true features which is unknown, we regard the original
$d'$ number of pre-processed features in the dataset as ``true'' features and then we augment the dataset with $30$ fake features. This problem
was studied by \citet{yamada2018post}. We apply our proposal to three datasets.

\textbf{Pulsar dataset} ($n=100,\ d'=8$) of \cite{lyon2016fifty} contain samples of pulsar candidates collected during the High Time Resolution Universe Survey. We split the dataset into two sets where one is for pulsars and the other for not pulsars. 

\textbf{Heart dataset} $(n=138,\ d'=13)$ of \cite{janosiuci} contains samples of patients, their attributes (such as age and sex) and whether they suffer from
heart disease. We split the dataset by whether they have heart disease or not.

\textbf{Wine dataset} $(n=100,\ d'=12)$ of \cite{cortez2009modeling} contains samples
related to red and white variants of the Portuguese ``Vinho Verde'' wine. The dataset is split into red and white wines.

\begin{table}
\centering
\begin{tabular}{@{}lllll@{}}
\toprule
        & \multicolumn{2}{c}{PolyMMD} & \multicolumn{2}{c}{MultiMMD} \\ \midrule
Dataset & TPR          & FPR          & TPR          & FPR           \\ \midrule
Pulsar  & 0.746        & 0.063        & 0.993         & 0.056         \\
Heart   & 0.359        & 0.042        & 0.588        & 0.049              \\
Wine    & 0.567        & 0.054        & 0.749        & 0.057           \\ \bottomrule
\end{tabular}
\caption{Benchmarking experiment using $\mathrm{MMD}_{Inc}$. The results are averaged over $100$ trials ($\alpha=0.05$).}
\label{table:feat}
\end{table}
The results are shown in Table \ref{table:feat}. It can be seen that TPR of our proposed method is higher than
PolySel for all datasets while both methods corroborate with the theory that the FPR is controlled at $\alpha$ in
all scenarios.
\subsection{Anomalous Dataset Detection}
In this experiment, we are given $6$ datasets with one desired reference set and our goal is to eliminate the datasets that deviate too far from the reference. To be specific, our datasets are formed from the smiling subset of the CelebA dataset \citep{liu2015faceattributes}, it may also contain synthetic samples generated from the smiling GAN of \cite{jitkrittum2018informative}. Instead of testing on raw pixels, the datasets are
pre-processed and represented by $2048$-dimensional features extracted from the
Pool3 layer of Inception-v3 \citep{szegedy2016rethinking}. Each dataset contains
$1000$ samples with $x\%$ being fake images and $1-x\%$ real images. In this
case, since all models are wrong \citep{box1976science}, 
the higher percentage of the presence of synthetic samples, the higher the
chance of rejection. We apply \textrm{MultiSel} and \textrm{PolySel} with the
IMQ kernel \citep{gorham2017measuring}.
\begin{table}[ht]
\centering
\begin{tabular}{@{}llcllll@{}}
\toprule
\multicolumn{1}{c}{Dataset}  & \multicolumn{1}{c}{1} & 2      & \multicolumn{1}{c}{3} & \multicolumn{1}{c}{4} & \multicolumn{1}{c}{5} & \multicolumn{1}{c}{6} \\ \midrule
\multicolumn{1}{c}{$\%$ Fakes} & $0\%$                 & $10\%$ & $20\%$                & $30\%$                & $40\%$                & $50\%$                \\ \midrule
MultiMMD                     & 0.03                  & 0.02   & 0.07                  & 0.06                  & 0.27                  & 0.49                  \\
PolyMMD                      & 0.02                  & 0.02   & 0.05                  & 0.04                  & 0.28                  & 0.45                  \\ \bottomrule
\end{tabular}
\caption{Rejection rate of five datasets for both MultiSel and PolySel. Each dataset has its own percentage of fake features.
These results were averaged over $100$ trials and we set $k=4$, $\alpha =0.05$ and $n=2000$.}
\label{table:anomaly}
\end{table}

The results are shown in Table \ref{table:anomaly}. The rejection rates of both methods are similar. Dataset 1 has the same distribution as our reference model and so the rejection rate of less than $\alpha$. As for the other datasets, the rejection rate increases as the percentage of fake increases but the similarity in the performance is expected and can be explained by the small difference in the selection event for PolyMMD and MultiMMD.

%% file: disc.tex


%% file: appendix.tex
\section{TRUE POSITIVE RATE (TPR) AND FALSE POSITIVE RATE (FPR)}
\label{sec:tpr_fpr}
Let $\mathcal{I}_-$ be the indices of features such that the null holds, i.e., for MMD, we have $\mathcal{I}_-:=\{i : \mathrm{MMD}(P^{(i)},Q^{(i)}) = 0\}$ (and
for HSIC, we have $\mathcal{I}_-:=\{i : \mathrm{HSIC}(P^{(i)},Q) = 0\}$). Similarly, let $\mathcal{I}_+$ be the indices of features such that the alternative holds, i.e., for MMD, we have $\mathcal{I}_+:=\{i : \mathrm{MMD}(P^{(i)},Q^{(i)}) > 0\}$ (and
for HSIC, we have $\mathcal{I}_+:=\{i : \mathrm{HSIC}(P^{(i)},Q) > 0\}$). Then, for a set of selected features $\mathcal{S}_k$ we define FPR and TPR as follows,
\begin{align*}
    &\mathrm{FPR} = \mathbb{E}\bigg [\frac{|\mathcal{S}_k \cap \mathcal{I}_- \cap \mathcal{R}|}{|\mathcal{S}_k \cap \mathcal{I}_-|}\bigg ],
    &\mathrm{TPR} = \mathbb{E}\bigg [\frac{|\mathcal{S}_k \cap \mathcal{I}_+ \cap \mathcal{R}|}{|\mathcal{S}_k \cap \mathcal{I}_+|}\bigg ],
\end{align*}
where $\mathcal{R}$ is the set of indices that the algorithm rejections and note that $\mathcal{R} \subseteq \mathcal{S}_k$.

\section{EMPIRICAL DISTRIBUTIONS OF $\widehat{\mathrm{MMD}}_{\text{Inc}}(X,Y)$
and $\widehat{\mathrm{HSIC}}_{\text{Inc}}(Z)$}
\label{sec:append_emp_dist}
In this section, we simulate the empirical distribution of the incomplete estimator for both
$\widehat{\mathrm{MMD}}_{\text{Inc}}(X,Y)$ and $\widehat{\mathrm{HSIC}}_{\text{Inc}}(Z)$.
\subsection{Empirical distribution of $\widehat{\mathrm{MMD}}_{\text{Inc}}(X,Y)$}
\textbf{Case $P=Q$:} For
MMD, we let $X \sim \mathcal{N}(0,1)$ and $Y \sim \mathcal{N}(0,1)$ which means 
$\widehat{\mathrm{MMD}}_{\text{u}}(X,Y)$ is degenerate whereas we show that 
$\widehat{\mathrm{MMD}}_{\text{Inc}}(X,Y)$ follows a normal distribution (see Figure \ref{fig:emp_mmdinc}).
When the $r$ is small, the empirical distribution of the incomplete estimators follows a normal distribution butas
$r$ gets bigger we expect it to behave like its complete estimator counterpart.
\begin{figure}[ht]
    \centering
    \subfloat[$r=1$]{
        \includegraphics[width=0.3\linewidth]{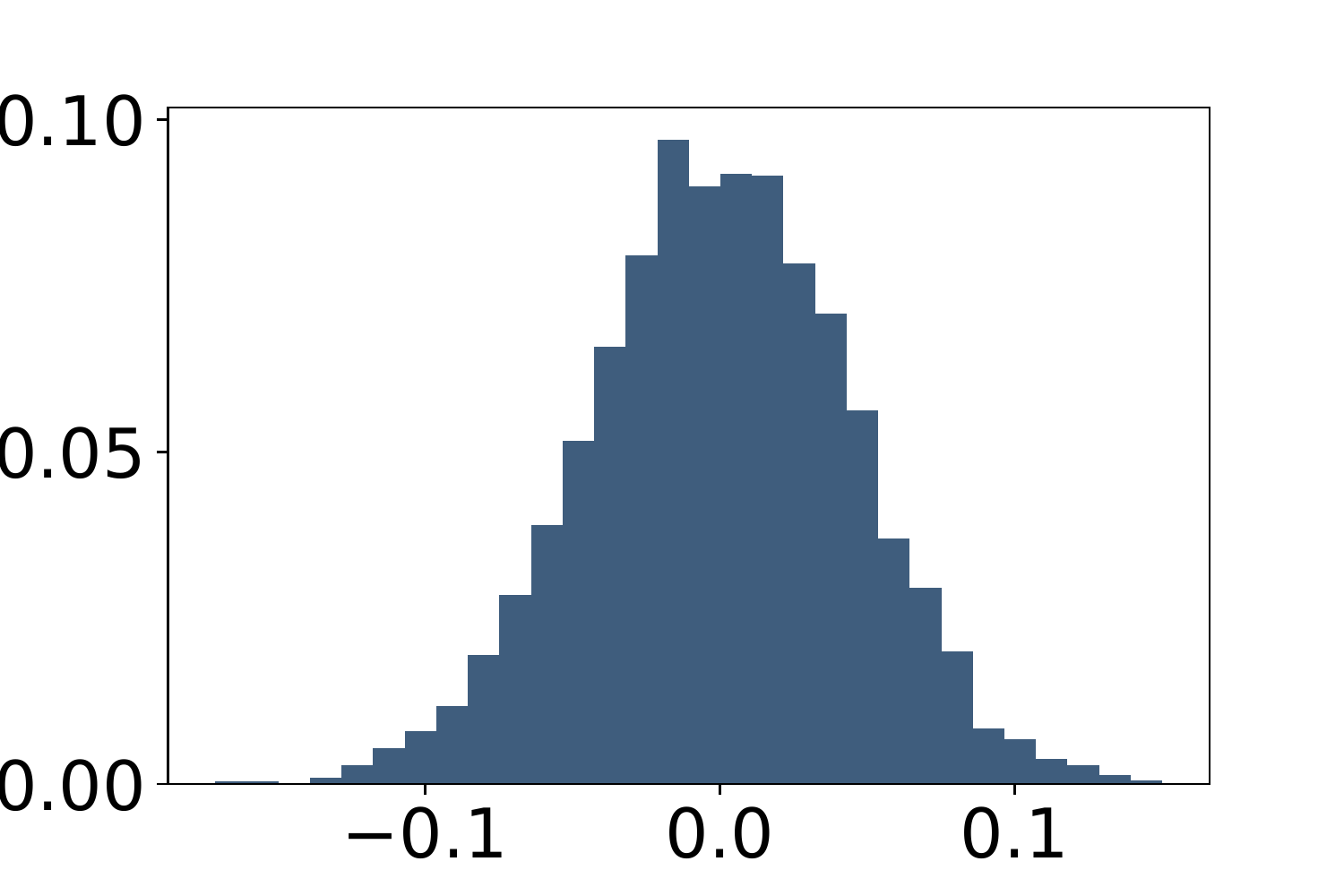}
    }
    \subfloat[$r=10$]{
        \includegraphics[width=0.3\linewidth]{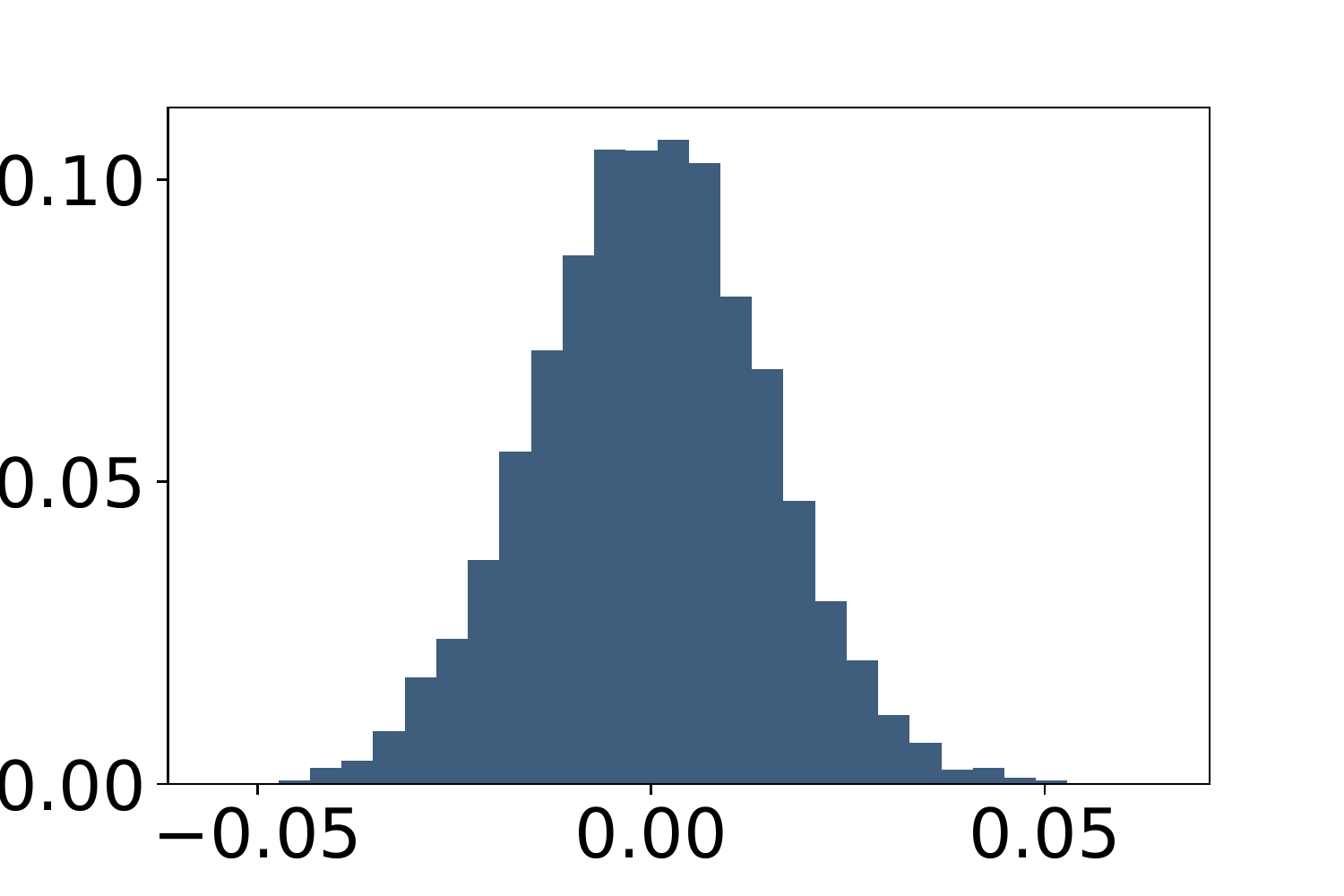}
    }
    \subfloat[$r=100$]{
        \includegraphics[width=0.3\linewidth]{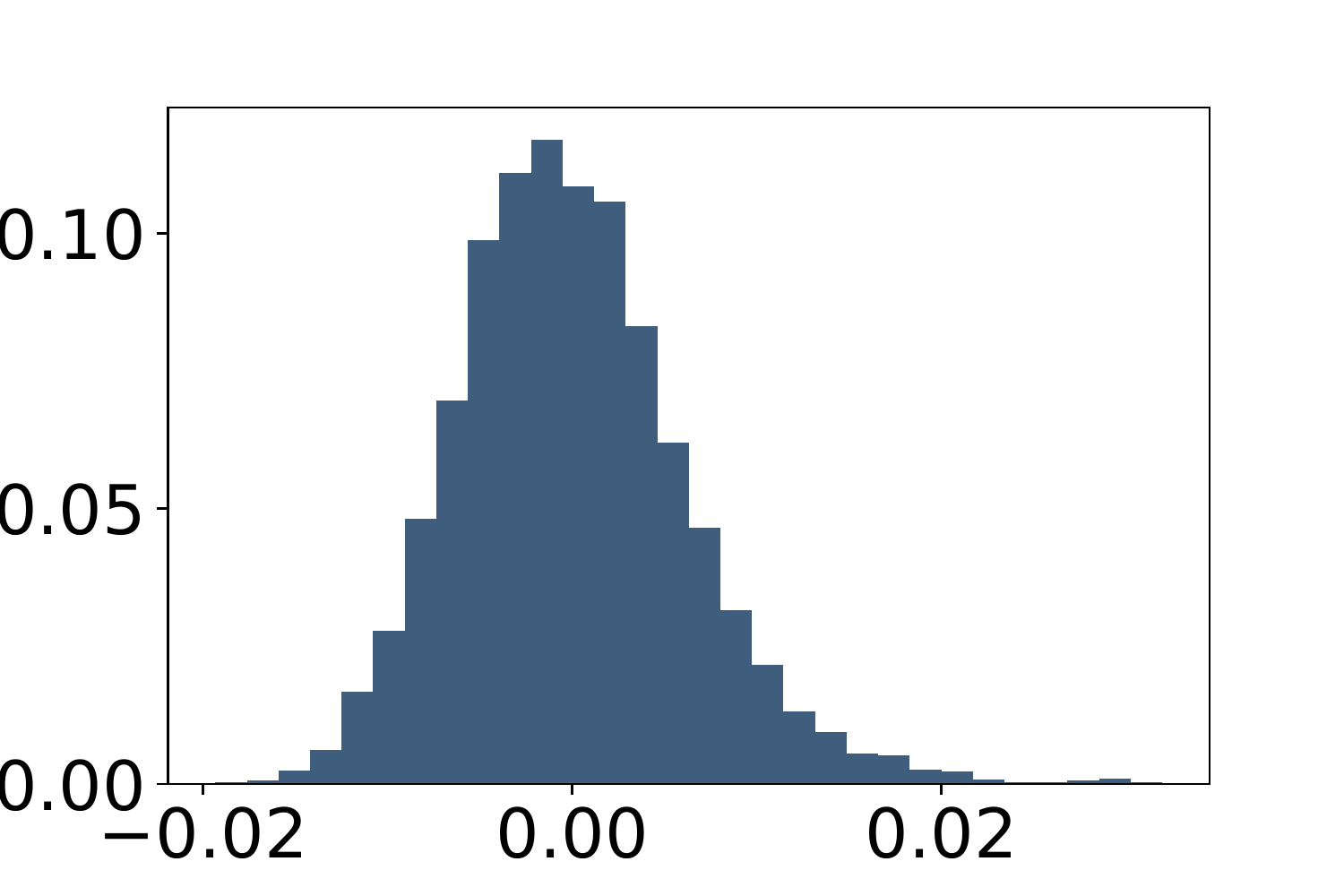}
    }
    \caption{The empirical distribution $\widehat{\mathrm{MMD}}_{\text{Inc}}(X,Y)$ for $r \in \{1, 10, 100\}$. $5000$ samples were used.}
    \label{fig:emp_mmdinc}
\end{figure}

\textbf{Case $P\neq Q$:} We show the empirical distribution of the incomplete estimator for MMD  when 
$P=\mathcal{N}(0,1)$ and $Q=\mathcal{N}(\mu,1)$ and $\mu \in \{ 0, 2, 3 \}$. Under the alternative, for our choice in $r$,
the distribution under the alternative is expected to have higher variance than the null distribution.

\begin{figure}[!ht]
    \centering
    \subfloat[$r=1$]{
        \includegraphics[width=0.3\linewidth]{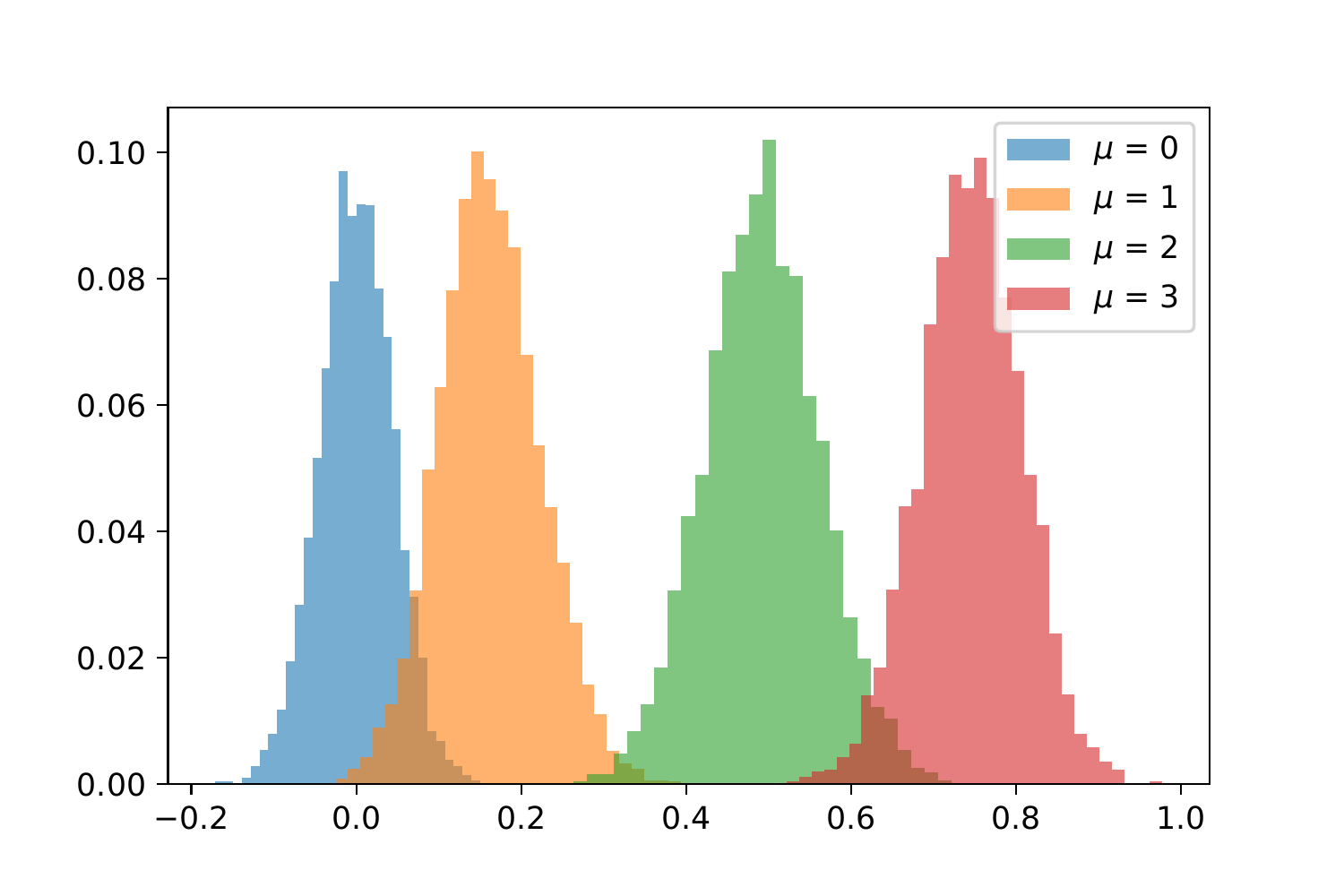}
    }
    \subfloat[$r=10$]{
        \includegraphics[width=0.3\linewidth]{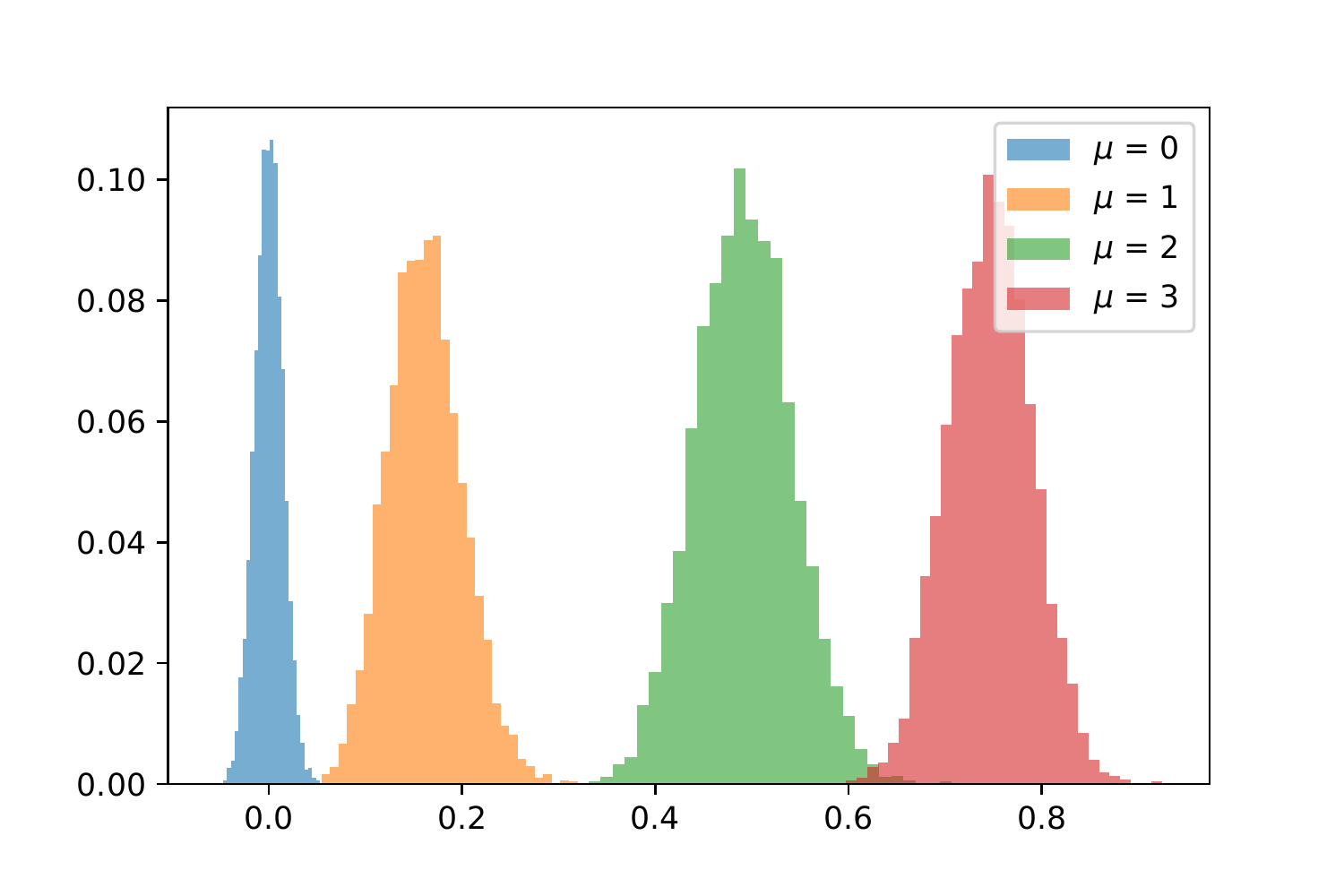}
    }
    \subfloat[$r=100$]{
        \includegraphics[width=0.3\linewidth]{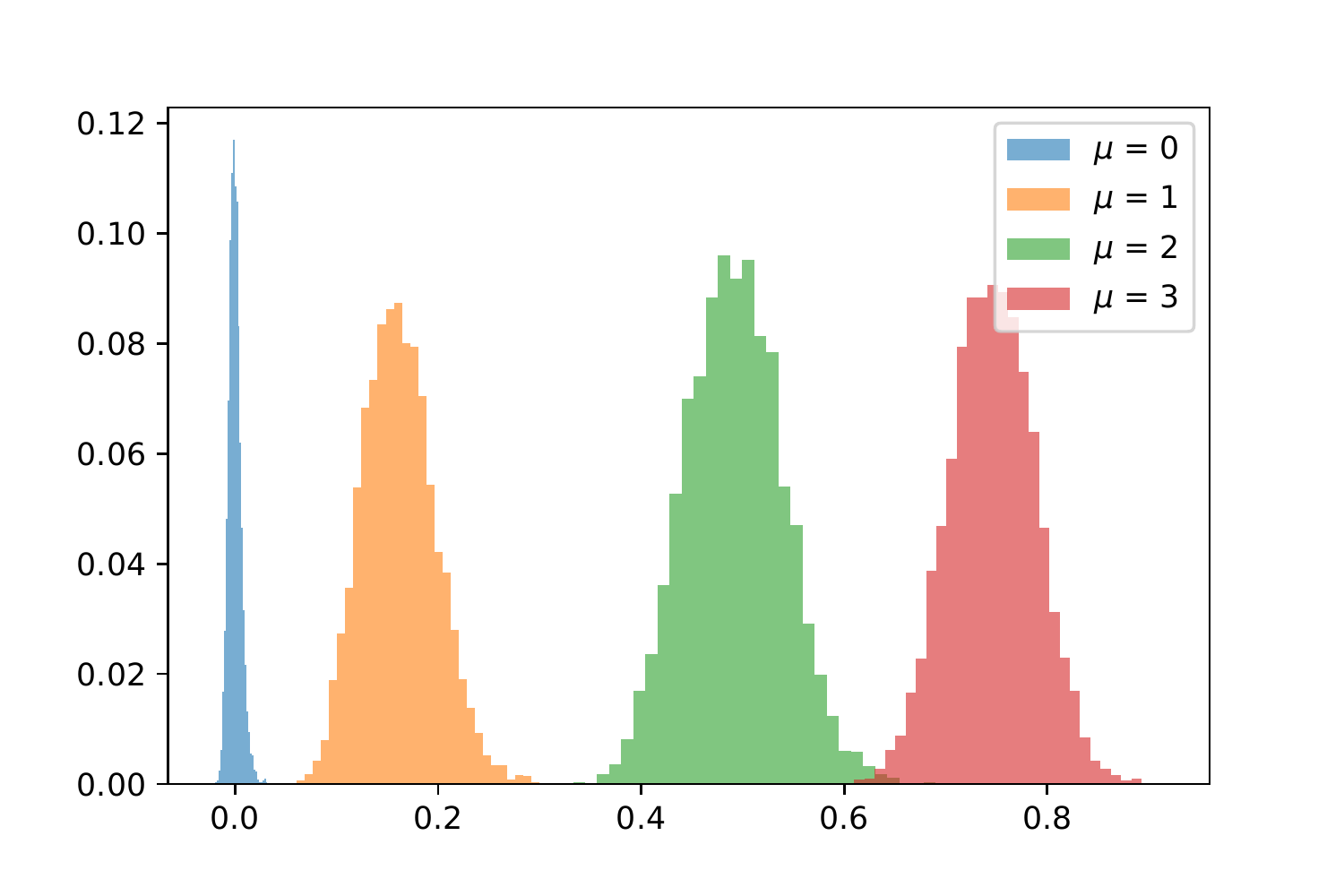}
    }
    \caption{The empirical distribution $\widehat{\mathrm{MMD}}_{\text{Inc}}(Z)$ for $r \in \{1, 10, 100\}$. $5000$ samples were used.}
\end{figure}

\subsection{Empirical distribution of $\widehat{\mathrm{HSIC}}_{\text{Inc}}(Z)$}

For HSIC,
let $Z:=(X,Y)$ where $X$ and $Y$ is follows a standard normal and is sampled independently of each other.
We show that in this case $\widehat{\mathrm{HSIC}}_{\text{Inc}}(Z)$ is also normal (see Figure \ref{fig:emp_hsicinc}).

\begin{figure}[!h]
    \centering
    \subfloat[$r=1$]{
        \includegraphics[width=0.3\linewidth]{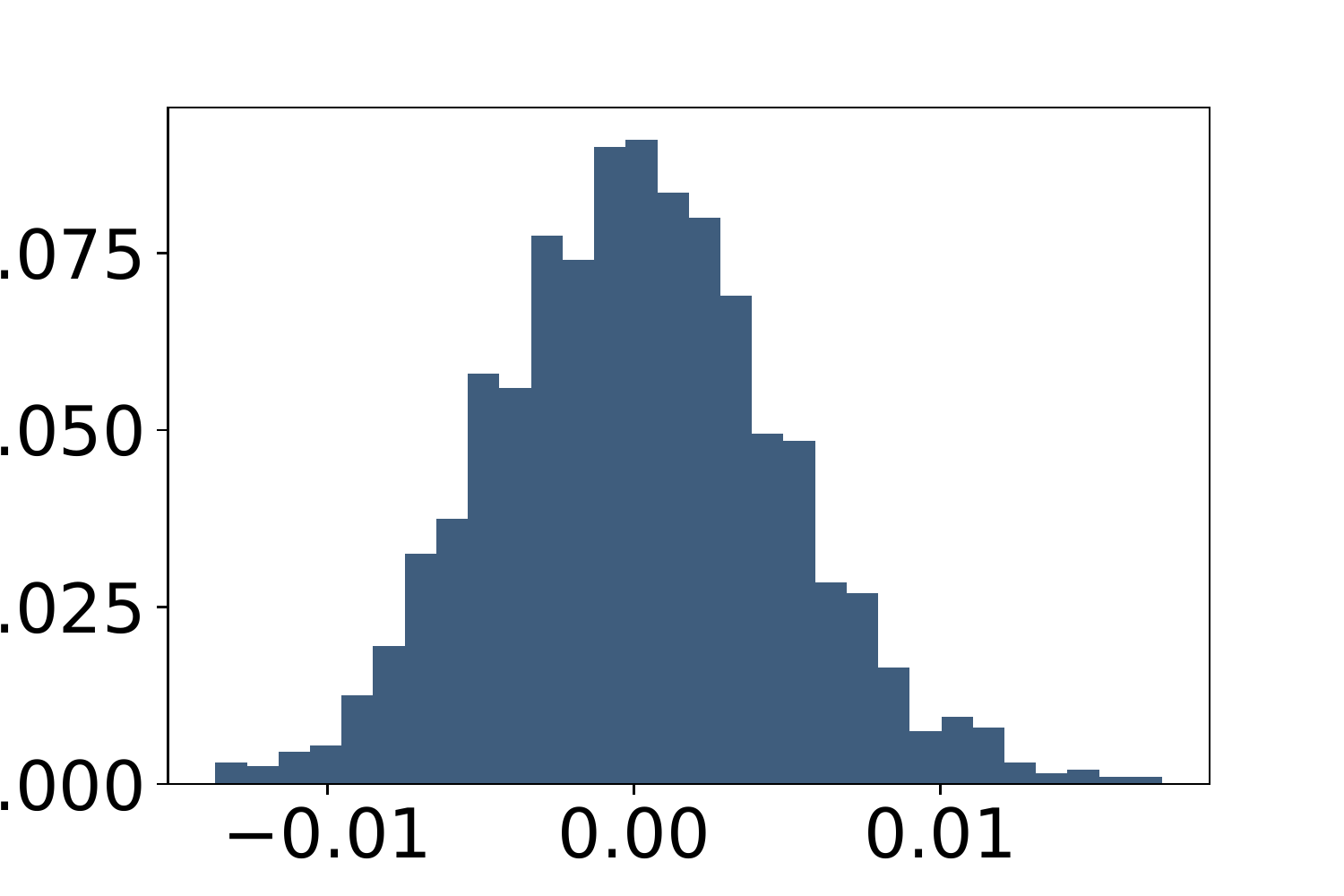}
    }
    \subfloat[$r=10$]{
        \includegraphics[width=0.3\linewidth]{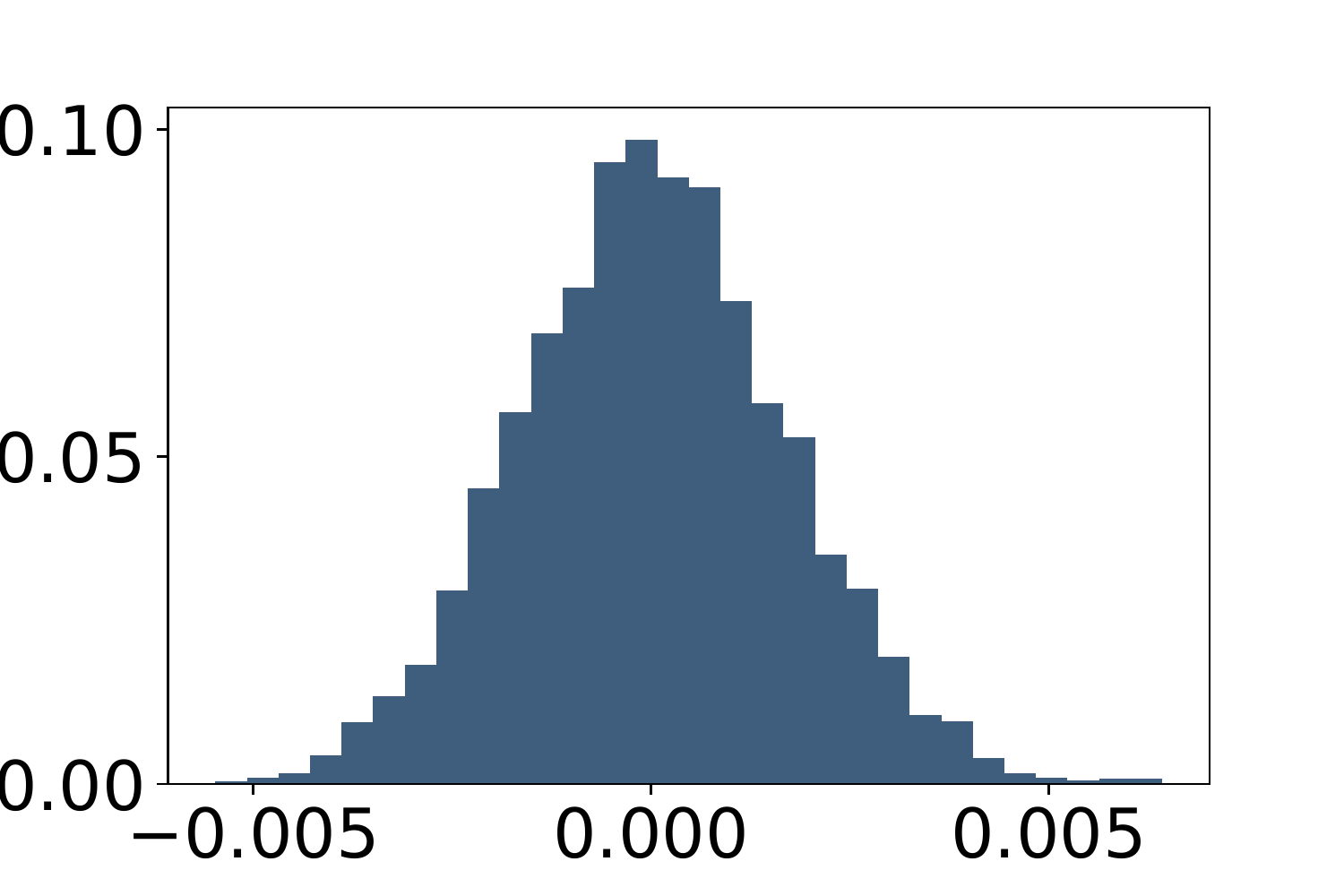}
    }
    \subfloat[$r=100$]{
        \includegraphics[width=0.3\linewidth]{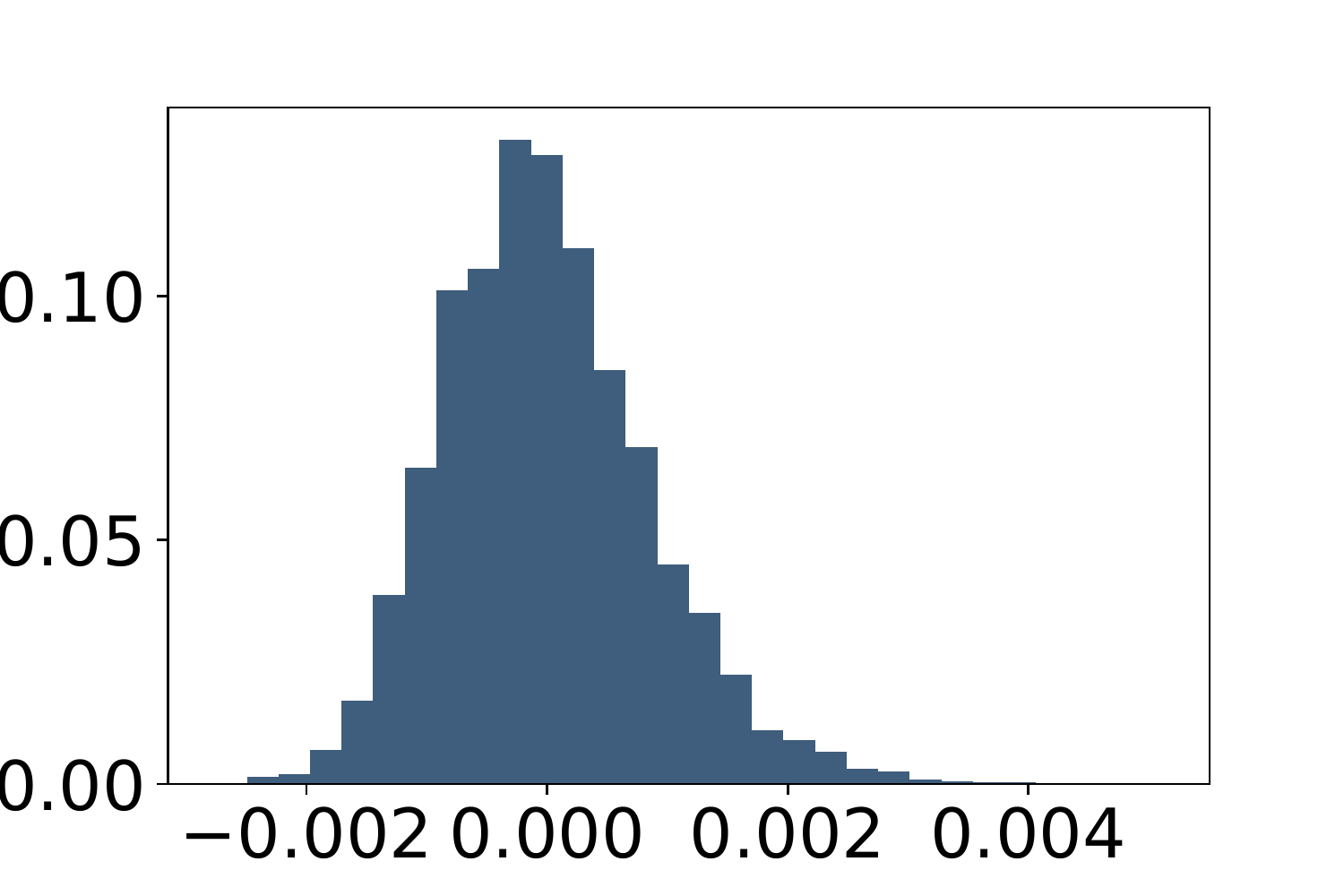}
    }
    \caption{The empirical distribution $\widehat{\mathrm{HSIC}}_{\text{Inc}}(X,Y)$ for $r \in \{1, 10, 100\}$. $5000$ samples were used.}
    \label{fig:emp_hsicinc}
\end{figure}

\section{MULTISCALE BOOSTRAP ALGORITHM FOR HSIC}
\label{sec:appen_multihsic}
In this section, we present algorithms for MultiHSIC for incomplete HSIC (Section \ref{sec:appen_multiinchsic}) and for block HSIC (Section \ref{sec:appen_multiblohsic}). Algorithm \ref{alg:multihsic} describes the procedure for calculating $p$-values using multiscale bootstrap.
\subsection{Incomplete HSIC}
\label{sec:appen_multiinchsic}
The parameters $\bm{T}(\bm{Z})$ and $\bm\Sigma$ for the incomplete estimator are estimated with the same method as for the incomplete MMD (see Section \ref{sec:si_mmd}). The algorithm is described in Algorithm \ref{alg:multihsic}.
\begin{algorithm}[ht]
    \caption{MultiHSIC($\bm{Z}_n, k,\, \mathcal{M}$): Selective $p$-values for the null hypothesis $H_{0,i}: \mathrm{HSIC}(P_{\bm{x}^{(i)}\bm{y}}) = 0\ |\ i \in \mathcal{S}_k \text{ is selected}$.}
    \label{alg:multihsic}
    \begin{algorithmic}[1] 
            \State $\hat{\bm{T}}(\bm{Z}), \hat{\bm\Sigma} \gets $EstimateParam($\bm{Z}_n$)
            \State $\mathcal{S}_k \gets $ the indexes
        of $k$ largest values of $\{\widehat{\mathrm{HSIC}}( \bm{Z}_n^{(i)})\}_{i\in\mathcal{I}}$
            \For{$i \in \mathcal{S}_k$}
                \For{$n' \in \mathcal{M}$}
                    \State $\gamma^2_{n'} \gets \frac{n}{n'}$
                    \State Sample $\{\bm{y}^*_i\}^B_{i=1} \overset{i.i.d.}{\sim} \mathcal{N}(\hat{\bm{T}}(\bm{Z}),\, \gamma^2_{n'}\hat{\bm\Sigma})$
                    \State $\bs_{\gamma_{n'}^2}(S) \gets \sum^B_{i=1}\mathds{1}^{(i)}_S(\bm{y}^*_i)/B$
                \EndFor
                \State Fit a linear model $\varphi_S(\gamma^2)$ such that $\varphi_S(\gamma^2)=\gamma\bar{\Phi}^{-1}(\bs_{\gamma^2}(S))$.
                \State $\hat{\beta}_0^{(i)} \gets \hat{\sigma}_i^{-1}\sqrt{l_n}\widehat{\mathrm{HSIC}}_{Inc}(\bm{Z}_n^{(i)})$
                \State $p_i \gets \Bar{\Phi}(\hat{\beta}_0^{(i)})/\Bar{\Phi}(\hat{\beta}_0^{(i)}+\varphi_S(0))$
            \EndFor
            \State \Return $\{p_i\}_{i=0}^k$ and $\mathcal{S}_k$
    \end{algorithmic}
\end{algorithm}

The following theorem justifies our use of the multivariate normal model,

\begin{restatable}[]{theorem}{asymptinchsictest}
Assume that $\lim_{n,l\rightarrow \infty} n^{-1}l=\lambda$ and assume that $\lim_{n,l\rightarrow \infty} n^{-2}l = 0$ and $0<\lambda<\infty$ then,
$    l^{\frac{1}{2}}\bigg ( \begin{bmatrix} \widehat{\mathrm{HSIC}}_{Inc}(\bm{Z}^{(1)}) \\ 
                                            \vdots \\ 
                                            \widehat{\mathrm{HSIC}}_{Inc}(\bm{Z}^{(d)})
    \end{bmatrix} -
    \begin{bmatrix} \mathrm{HSIC}(P_{\bm{x}^{(1)}\bm{y}}) \\ 
                                            \vdots \\ 
                                            \mathrm{HSIC}(P_{\bm{x}^{(d)}\bm{y}})
    \end{bmatrix}
    \bigg )$
is asymptotically normal.
\end{restatable}
The proof can be found in Appendix \ref{proof:hsicstat}. 
\subsection{Block HSIC}
\label{sec:appen_multiblohsic}
\textbf{Block estimator as the incomplete estimator:} The block estimator $\mathrm{\widehat{HSIC}}_{Blo}$ 
\citep{zhang2018large} is an example of an incomplete estimator for HSIC with a fixed design matrix. To see this note that for a given blocksize $B$, we have a total of $\frac{n}{B}$ blocks. For each block, the complete U-statistic estimator is calculated, i.e., for block $t$
$$
\hat{\eta}(t) = \frac{(B-4)!}{B!}\sum_{(i,j, q,r) \in \mathbf{i}^{[(t-1)B+1,tB]}_4} h(i,j, q,r),
$$
where $\mathbf{i}^{[u,i]}_4$ is the set of 4-tuple with each index, between $u$ and $i$, appearing exactly once. There are a total of $\frac{n}{B}$ blocks that are averaged to produce $\mathrm{\widehat{HSIC}}_{Blo}$, i.e., we have
$$
\mathrm{\widehat{HSIC}}_{Blo} =  \frac{B}{n}\sum_{t=1}^{\frac{n}{B}}\hat{\eta}(t).
$$
 Thus, we have shown that $\mathrm{\widehat{HSIC}}_{Blo}$ can be rewritten as $\widehat{\mathrm{HSIC}}_{Inc}$ where we have  $\mathcal{D} = \cup_{t=1}^{\frac{n}{B}} \mathbf{i}^{[(t-1)B+1,tB]}_4$. Note that $|\mathcal{D}_{Blo}| = \frac{(B-1)!}{(B-4)!}n$.
\label{sec:appen_bias}

\textbf{Algorithm:}
The extension to multiscale bootstrap to include the block estimator is simple. It only requires changes in the parameters of the resampling distribution for varying $n'$, as a well as how the signed distance $\hat{\beta}^{(i)}_0$ for feature $i$ is calculated.

Let $\hat{\bm{T}}(\bm{Z}) := \sqrt{\frac{n}{B}}[\widehat{\mathrm{HSIC}}_{Blo}(\bm{Z}_n^{(1)}),\dots,\widehat{\mathrm{HSIC}}_{Blo}(\bm{Z}_n^{(d)})]^\top$ and $\bm{T}(\bm{Z})$ be its population counterpart, namely, $\bm{T}(\bm{Z}) = \sqrt{\frac{n}{B}}[{\mathrm{HSIC}}(P_{\bm{x}^{(1)}\bm{y}}),\dots,{\mathrm{HSIC}}(P_{\bm{x}^{(d)}\bm{y}})]^\top$. 
Note that $\hat{\bm{T}}(\bm{Z})$ can be equivalently written as $\sum_{i=1}^{n/B}\hat{\bm\eta}(i)$ where
$\hat{\bm\eta}(i)=[\hat{\eta}^{(1)}(i),\dots,\hat{\eta}^{(d)}(i)]^\top$, and $\hat{\eta}^{(j)}(i)$ is the complete U-statistic estimator for HSIC applied to the $i$-th block of $\bm{Z}^{(j)}$. Then in the limit
$n\rightarrow \infty$, $B \rightarrow \infty$, and $\frac{n}{B} \rightarrow \infty$ \citep{zhang2018large}, we have
under the null hypothesis
$$
\hat{\bm{T}}(\bm{Z})-\bm{T}(\bm{Z}) \sim \mathcal{N}(\bm{0},\,\bm\Sigma),
$$
where $\bm\Sigma$ is the covariance matrix with its elements as $\bm\Sigma_{ij} = \mathrm{Cov}\left(\bm\eta^{\left(i\right)},\bm\eta^{\left(j\right)}\right)$.
We estimate $\bm\Sigma$ with the sample covariance $\hat{\bm{\Sigma}}$, i.e., we have
$\hat{\bm\Sigma}:=\frac{B}{n}\sum_{i =1}^{n/B} [\hat{\bm\eta}(i) - \overline{\bm\eta}][\hat{\bm\eta}(i) - \overline{\bm\eta}]^\top$. Then for varying $n'$, instead of resampling $n'$ samples from $\bm{Z}$, we produce samples directly from $\mathcal{N}(\hat{\bm{T}}(\bm{Z}), \frac{n}{n'}\hat{\bm\Sigma})$ as before.
The sign distance $\hat{\beta}_0^{(i)}$ is $ \hat{\sigma}_i^{-1}\sqrt{\frac{n}{B}}\widehat{\mathrm{HSIC}}_{Blo}(\bm{Z}_n^{(i)})$ where $\hat{\sigma}_i^{-1}$ is the $i$-th diagonal element of $\hat{\bm\Sigma}$.

%
\begin{wrapfigure}{r}{80mm}
\centering
\includegraphics[width=80mm]{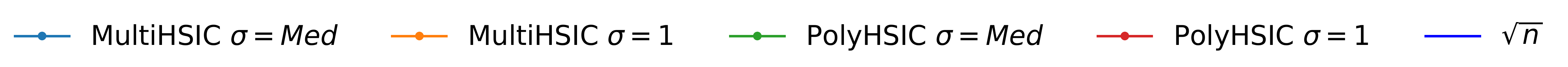}
\subfloat{\includegraphics[width=40mm]{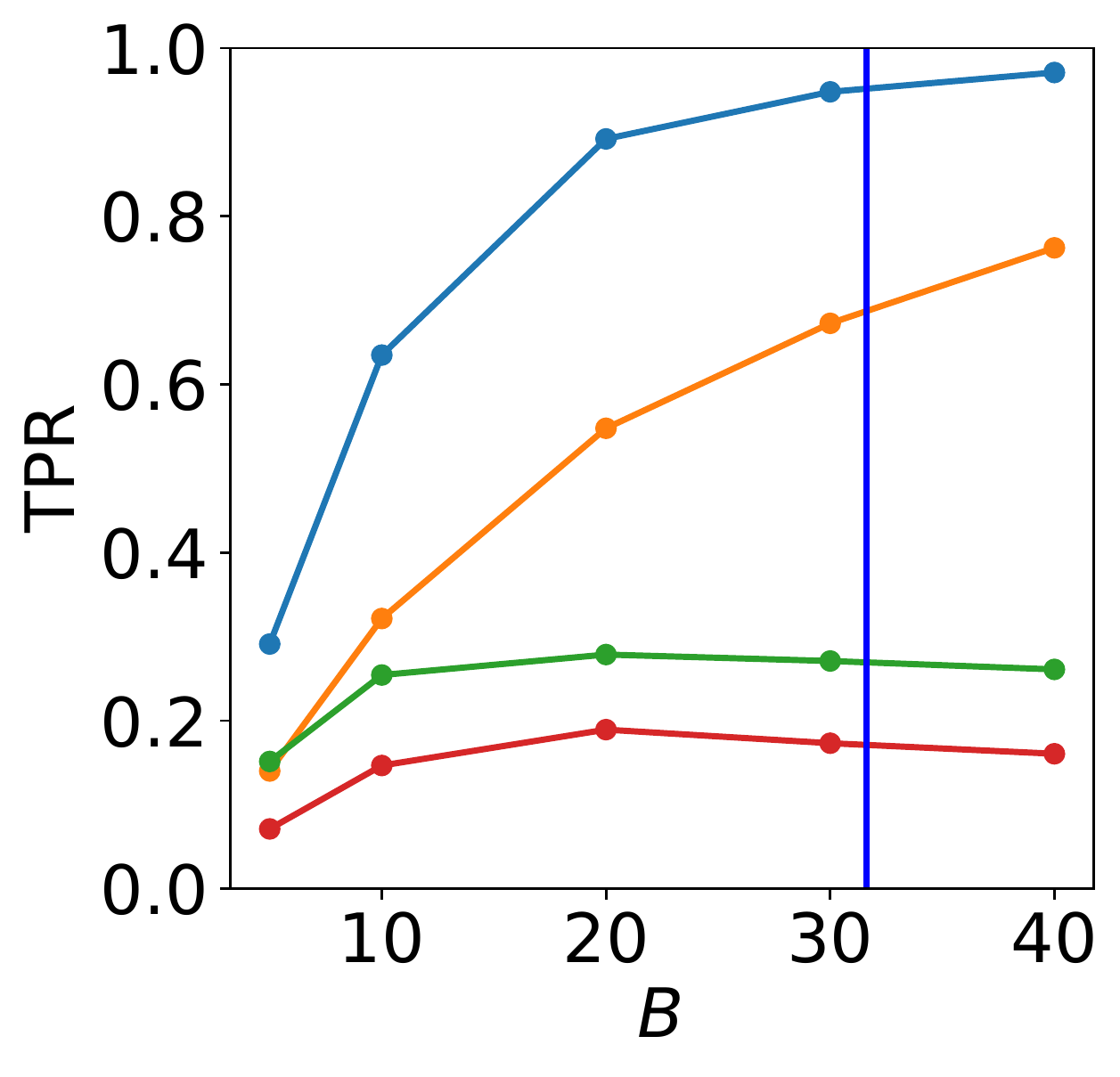}}
\subfloat{\includegraphics[width=40mm]{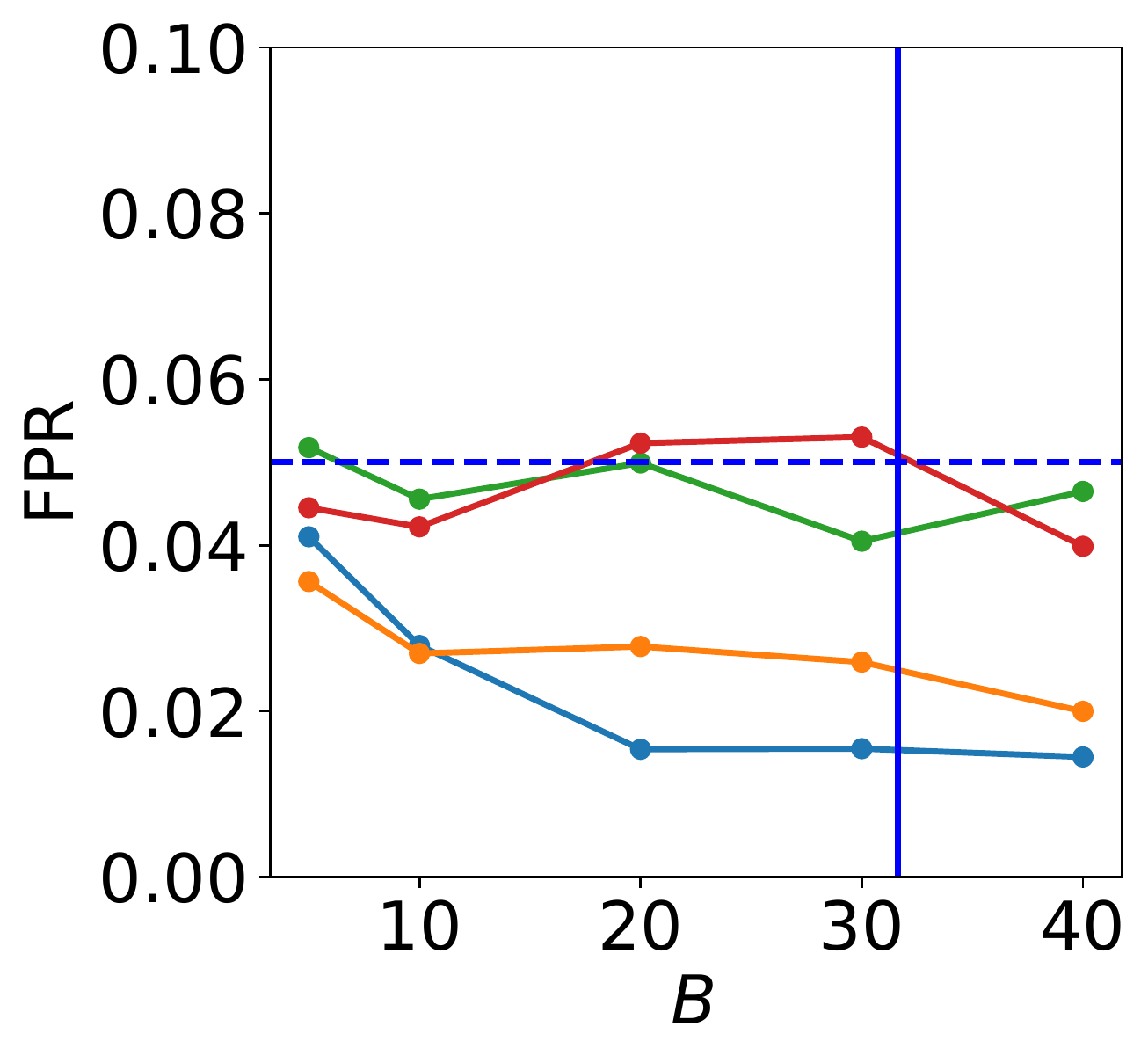}}
\caption{Logistic experiment. $B$ increases for $\mathrm{\widehat{HSIC}}_{B}$. We use a Gaussian kernel with its bandwidth either set to be $1$ or chosen with the median heuristic. We use $n=1000$.}
\label{fig:ex7}
\end{wrapfigure}
\textbf{Empirical Results:}
In this experiment, we use the same setup as Figure \ref{fig:ex3} for the Logit problem and the results are shown in Figure \ref{fig:ex7}. Our aim is to investigate the behaviour of our test when $B$ the block size increases. In \citet[Section 5]{zaremba2013b}, they investigated the behaviour of the block estimator under finite samples and found that there can have severe bias under the null hypothesis.

In our results, we observed that there was a large deviation for the nominal size $\alpha$ and an increase in the TPR. We speculate that this is due to the positive bias in finite samples of the skewness of the block estimator. These experiments show that the effect is more pronounced for MultiHSIC (than PolyHSIC) which may be because of our choice in parameterising the bootstrap samples as a normal distribution. We note that the effect of FPR going below the nominal $\alpha$ is not just for very large values of $B$ but even for the recommended heuristic $B=\sqrt{n}$. It would be interesting to investigate this problem and correct for it in future works.

\section{PROOFS}
\label{sec:appen_hsic}
In this section, we provide proofs for our statements in Section \ref{sec:si_hsic}. Before we begin, recall that

$$
h(i,j,q,r) = \frac{1}{4!}\sum_{(s,t,u,v)}^{(i,j,q,r)}\bm{K}_{st}[\bm{L}_{st}+\bm{L}_{uv}-2\bm{L}_{su}]
$$
is the order-$4$ U-statistic kernel for HSIC. We define the conditional expectation of
the U-statistic kernel
\begin{align*}
    h_4 &= h(i,j,q,r),\\
    h_3 &= \mathbb{E}[h(i,j,q,r)\,|\,i,\,j,\,q\,],\\
    h_2 &= \mathbb{E}[h(i,j,q,r)\,|\,i,\,j\,] \\
    h_1 &= \mathbb{E}[h(i,j,q,r)\,|\,i\,].
\end{align*}

Let $c$ be the smallest integer such that $h_c \neq \mathrm{HSIC}$. When $P \independent Q$, we have $h_1 = 0$ and $\mathrm{HSIC}=h_1$ so $c > 1$. However when $P \notindependent Q$, $h_1 \neq \mathrm{HSIC}$ so $c=1$.
Similarly, we  show that $\widehat{\mathrm{HSIC}}_{Inc}$ is asymptotically normal under mild assumptions.

\begin{restatable}[Asymptotic Distribution of $\widehat{\mathrm{HSIC}}_{Inc}$]{theorem}{asymptinchsic}
\label{thm:asympt_inchsic1}
Let $c$ be the smallest integer such that $h_c \neq \mathrm{HSIC}$ ($h_c$ defined in Appendix \ref{sec:appen_hsic}) and let $\lim_{n,l \rightarrow \infty} n^{-c}l=\lambda$
$(0 \le \lambda \le \infty)$ and let 
$\mathcal{D}$ be constructed by selecting $l$ subsets with replacement from $\bm{i}^{n}_4$ then,
\begin{enumerate}
    \item If $\lambda = 0$ then, $l^{\frac{1}{2}}(\widehat{\mathrm{HSIC}}_{\text{Inc}}(\bm{z}) -\mathrm{HSIC}(P_{xy}) \overset{d}{\rightarrow} \mathcal{N}(0,\sigma^2)$,
    \item If $0 < \lambda < \infty$ then, $l^{\frac{1}{2}}(\widehat{\mathrm{HSIC}}_{\text{Inc}}(\bm{z}) -\mathrm{HSIC}(P_{xy})) \overset{d}{\rightarrow} \lambda^{\frac{1}{2}} V + T$,
    \item If $\lambda = \infty$ then, $n^{\frac{c}{2}}(\widehat{\mathrm{HSIC}}_{\text{Inc}}(\bm{z}) -\mathrm{HSIC}(P_{xy})) \overset{d}{\rightarrow} V$,
\end{enumerate}
where $V$ is a random variable with the limit distribution of $n^{c/2}(\widehat{\mathrm{HSIC}}_{\text{u}}(\bm{z})-\mathrm{HSIC})$ and $T \sim \mathcal{N}(0,\sigma^2)$ where $\sigma^2=\mathrm{Var}[h(i,j,q,r)]$.
\end{restatable}

\begin{proof}
See \citet[Corollary 1]{janson1984asymptotic} and \citet[Theorem 1, Section 4.3.3]{lee2019u}
\end{proof}

\asymptinchsico*

\begin{proof}
When $P \independent Q$, then $c \ge 2$ then the result immediately
follows from Theorem \ref{thm:asympt_inchsic1} for the case $\lambda=0$.

For $P \notindependent Q$, then $c=1$ thus, under our assumptions, we obtain our result from Theorem \ref{thm:asympt_inchsic1}.
\end{proof}

\asymptinchsictest*

\begin{proof}
\label{proof:hsicstat}
This proof is identical to the proof of \citet[Theorem 5]{yamada2018post}. From Cram\'er-Wold theorem, it is sufficient to prove that for every $\bm\eta \in \mathbb{R}^d$, 
$$
\bm\eta^\top\begin{bmatrix} \widehat{\mathrm{HSIC}}_{Inc}(\bm{Z}^{(1)}) \\ 
                                            \vdots \\ 
                                            \widehat{\mathrm{HSIC}}_{Inc}(\bm{Z}^{(d)})
    \end{bmatrix}
    \overset{d}{\rightarrow}
    \bm\eta^\top\bm{V}
$$
where $\bm{V}$ is some normal distribution. Under our assumptions, for all $i$ $\widehat{\mathrm{HSIC}}_{Inc}(\bm{Z}^{(i)})$ follows a normal distribution. Following from the continuous mapping theorem, for all $\bm\eta \in \mathbb{R}^d$ we have as desired.
\end{proof}

\section{Additional Experiments}
In this section, we provide additional experiments with HSIC. The first is a benchmarking experiment similar to the one performed in Section \ref{sec:mmd_benchmark}. The second uses the Divorce dataset \citep{yontem2019divorce} where people were given a questionnaire about their marriage and asked to rate each statement about their marriage from 0 to 4 depending on the truthfulness. 
\label{sec:appen_exp}
\subsection{Benchmark}
The goal is to rediscover the original features with statistical significance. As seen in the Table \ref{ex:benchmark_hsic}, the results indicate that MultiSel achieves higher power (as with the MMD).
\begin{table}[ht]
\centering
\begin{tabular}{@{}lclcl@{}}
\toprule
               & \multicolumn{2}{c}{MultiSel-HSIC} & \multicolumn{2}{c}{PolySel-HSIC} \\ \midrule
Dataset        & TPR             & FPR            & TPR             & FPR             \\ \midrule
Pulsar ($n=100$) & $0.705$           & $0.023$          & $0.625$           & $0.025$           \\
Heart ($n=138$)  & $0.469$           & $0.029$         & $0.410$         & $0.030$           \\
Wine ($n=200$)   & $0.800$            & $0.042$          & $0.730$           & $0.058$           \\ \bottomrule
\end{tabular}
\caption{The TPR and FPR for the benchmarking experiment using $\widehat{\mathrm{HSIC}}_{Inc}$. The results are averaged over $100$ trials, with $\alpha = 0.05$.}
\label{ex:benchmark_hsic}
\end{table}

\subsection{Divorce Dataset}
We report the calculated $p$-values of each statistical test of dependency between a selected statement and the outcome of divorce. In the experiment, we chose $r=15$, $k=15$ (out of $54$) and $n=150$ with the results summaries in Table \ref{ex:divorce}. We found that MultiSel declared $6$ more statements as significantly (than PolySel) with a significance level at $\alpha=0.05$, including statements such as ``I feel aggressive when I argue with my wife.'' and ``My wife and most of our goals are common.''. We do not know the ground truth but the $6$ statements seem plausible. The results suggest that MultiSel has higher detection rate.

\begin{table}[ht]
\centering
\begin{tabular}{@{}lll@{}}
\toprule
                                                                                                                                        & \multicolumn{2}{c}{$p$-values}    \\ \midrule
                                                                                                                                        & MultiSel-HSIC   & PolySel-HSIC    \\ \midrule
My argument with my wife is not calm.                                                                                                   & \textless{}0.01 & \textless{}0.01 \\
Fights often occur suddenly.                                                                                                            & \textless{}0.01 & 0.41            \\
I can insult my spouse during our discussions.                                                                                          & \textless{}0.01 & 0.09            \\
\begin{tabular}[c]{@{}l@{}}When fighting with my spouse, I usually use expressions\\  such as ‘you always’ or ‘you never’.\end{tabular} & \textless{}0.01 & 0.17            \\
We're compatible with my wife about what love should be.                                                                                & \textless{}0.01 & 0.43            \\
My wife and most of our goals are common.                                                                                               & \textless{}0.01 & 0.25            \\
I feel aggressive when I argue with my wife.                                                                                            & \textless{}0.01 & 0.22            \\
We're starting a fight before I know what's going on.                                                                                   & \textless{}0.01 & \textless{}0.01 \\
\begin{tabular}[c]{@{}l@{}}I can use negative statements about my wife's personality\\  during our discussions.\end{tabular}            & \textless{}0.01 & 0.05            \\
I hate my wife's way of bringing it up.                                                                                                 & \textless{}0.01 & \textless{}0.01 \\
I enjoy our holidays with my wife.                                                                                                      & 0.12            & 0.27            \\
When we fight, I remind her of my wife's inadequate issues.                                                                             & 0.13            & 0.04            \\
\begin{tabular}[c]{@{}l@{}}When I argue with my wife, it will eventually work for me \\ to contact him.\end{tabular}                    & 0.16            & 0.14            \\
I know my wife's hopes and wishes.                                                                                                      & 0.77            & 0.56            \\
I can use offensive expressions during our discussions.                                                                                 & 0.94            & 0.89            \\ \bottomrule
\end{tabular}
\caption{The resultant $p$-values from one trial of the divorce dataset using $\mathrm{HSIC}_{Inc}$.}
\label{ex:divorce}
\end{table}